%% file: main_aistats2021.tex
\documentclass[twoside]{article}

\usepackage[accepted]{aistats2021}

\usepackage[round]{natbib}

\usepackage{url}
\usepackage{subcaption}

\usepackage{graphicx}
\usepackage{hyperref}%
\hypersetup{colorlinks=true, linkcolor=magenta, citecolor=blue}
\usepackage{algorithm, algorithmic}
\usepackage{amsmath,amsthm,amssymb}
\usepackage{bm}
\usepackage{color}
\input{head}

\begin{document}

\twocolumn[

\aistatstitle{Ridge Regression with Over-Parametrized Two-Layer Networks Converge to Ridgelet Spectrum}
\aistatsauthor{ Sho Sonoda \And Isao Ishikawa \And  Masahiro Ikeda }
\aistatsaddress{ RIKEN AIP \And  Ehime University \& RIKEN AIP \And RIKEN AIP } ]

\begin{abstract}
\input{aistats2021/abstract}
\end{abstract}

\input{aistats2021/body}

\subsubsection*{Acknowledgements}
We thank the anonymous reviewers for their careful reading of our manuscript and their many insightful comments and suggestions.
We thank Taiji~Suzuki and Atsushi~Nitanda for productive comments on improving this study in many directions.
This work was supported by JSPS KAKENHI 18K18113, JST CREST JPMJCR1913, JPMJCR2015, and JST ACTX JPMJAX2004.

\bibliography{summary_library}

\newpage
\onecolumn
\appendix
\input{aistats2021/supp}

\end{document}

%% file: head.tex
\newcommand{\refeq}[1]{\eqref{eq:#1}}%
\newcommand{\reffig}[1]{Figure~\ref{fig:#1}}

\newcommand{\refapp}[1]{Appendix~\ref{sec:#1}}

\newcommand{\refthm}[1]{Theorem~\ref{thm:#1}}

\newtheorem{thm}{Theorem}[section]
\newtheorem{lem}[thm]{Lemma}
\newtheorem{prop}[thm]{Proposition}
\newtheorem{dfn}[thm]{Definition}
\newtheorem{asm}[thm]{Assumption}
\newtheorem{cor}[thm]{Corollary}

\theoremstyle{definition}
\newtheorem{rmk}[thm]{Remark}

\newcommand{\RR}{\mathbb{R}}
\newcommand{\NN}{\mathbb{N}}
\newcommand{\TT}{\mathbb{T}}
\newcommand{\ZZ}{\mathbb{Z}}

\newcommand{\II}{\mathbb{I}}
\newcommand{\mcT}{\mathcal{T}}

\newcommand{\EE}{\mathbb{E}}
\newcommand{\dd}{\mathrm{d}}

\newcommand{\argmin}{\mathop{\rm arg~min}\limits}

\newcommand{\Fx}{L^2(\RR^m)}%
\newcommand{\Fp}{L^2(P)}%
\newcommand{\Git}{L^2(\mu_A)}%
\newcommand{\Grt}{L^2(\RR^m \times \TT)}%
\newcommand{\Gd}{L^2(\lambda_d)}%
\newcommand{\Sa}{S_{\mu_A}}%
\newcommand{\Sd}{S_d}%
\newcommand{\gstar}{\gamma^*}%
\newcommand{\gd}{\gamma_d}%
\newcommand{\loss}{J}%

\renewcommand{\aa}{{\bm{a}}}

\newcommand{\nn}{{\bm{n}}}
\newcommand{\xx}{{\bm{x}}}
\newcommand{\yy}{{\bm{y}}}
\newcommand{\zz}{{\bm{z}}}
\newcommand{\ww}{{\bm{w}}}
\newcommand{\xxi}{{\bm{\xi}}}

\newcommand{\iprod}[1]{\langle#1\rangle}

\newcommand{\proj}{\mathsf{Proj}}
\newcommand{\ind}{\chi}

\newcommand{\wto}{\xrightarrow{w.}}

\newcommand{\eps}{\varepsilon}

\newcommand{\sign}{\mathrm{sign}}
\newcommand{\relu}{\mathrm{relu}}

%% file: aistats2021/abstract.tex
Characterization of local minima draws much attention in theoretical studies of deep learning. In this study, we investigate the distribution of parameters in an over-parametrized finite neural network trained by ridge regularized empirical square risk minimization (RERM). We develop a new theory of ridgelet transform, a wavelet-like integral transform that provides a powerful and general framework for the theoretical study of neural networks involving not only the ReLU but general activation functions. We show that the distribution of the parameters converges to a spectrum of the ridgelet transform. This result provides a new insight into the characterization of the local minima of neural networks, and the theoretical background of an inductive bias theory based on lazy regimes. We confirm the visual resemblance between the parameter distribution trained by SGD, and the ridgelet spectrum calculated by numerical integration through numerical experiments with finite models.

%% file: aistats2021/body.tex
\section{INTRODUCTION}

Characterizing local minima is important in theoretical studies of neural networks.
Despite the high-dimensionality of parameters, neural networks have become state-of-the-art in many application areas since the emergence of AlexNet \citep{Krizhevsky2012}. This has been a mystery of machine learning theory because several VC-based arguments have shown that the generalization error is upper bounded by the dimension of parameters, or the capacity of the hypothesis class \citep{Neyshabur2015,Bartlett2017}, but as \citet{Arora2018a} pointed out, these bounds are not tight in practice.
As \citet{Zhang2017} suggested, many researchers now consider that the typical solutions obtained via deep learning are concentrated in a much smaller class than expected from the algebraic dimension of parameters or any other data-independent capacities.

However, characterizing local minima is a challenging problem due to the nonlinearity of parameters and the non-convexity of learning problems.
To tackle this problem, the \emph{over-parametrization} is considered to be one of the promising assumption for theoretical analysis of neural networks, which assumes that the number of parameters in neural networks is sufficiently larger than the sample size.
This assumption has revolutionized our understanding of the local minima.
For example, the global convergence of deep learning is now proved in many ways,
and some researchers further conjecture that the typical solutions are close to the initial parameters
(see Section \ref{sec: related works} for more details).

In this study, %
we provide an explicit expression for the \emph{global minimizer} in the over-parametrized regime by means of the integral representation \citep{Barron1993,Murata1996,sonoda2015}.
The integral representation is an effective machinery to analyze the neural networks using harmonic analysis, a branch of mathematics. %
It is realized as a linear operator between function spaces (see Definition \ref{def: integral representation}), and  
provides a principled approach to study over-parametrized neural networks with not only ReLU but also a wide range of activation functions. Recently, this has been recognized as an effective reparametrization in the \emph{mean-field theory} \citep{Mei2018,Rotskoff2018}, which employs the integral representation to show the global convergence for finite two-layer networks.

To be precise, we develop a new theory of the \emph{rigelet transform on the torus}, %
and prove for the first time that the parameter distributions of \emph{finite two-layer neural networks trained by regularized empirical risk minimization (RERM)} converges to a ridgelet spectrum as both the parameter number and sample size tend to infinity.
By virtue of the over-parametrization, our theorem holds not only for strict global minima but also other suboptimal minima such as random features solutions.
The ridgelet transform, which is a wavelet-like integral transform, is originally developed by \citet{Murata1996}, \citet{Candes.PhD} and \citet{Rubin.calderon}, and has a remarkable application to analysis of neural networks (see eg., \citealp{Starck2010} and Appendix \ref{sec:Rcalc}).

Numerical simulation confirms our main theoretical results. Namely, the scatter plot of parameter distributions learned by stochastic gradient descent (SGD) shows a similar pattern to the ridgelet spectrum. While our theory do not assume any specific training algorithm (but ERM), the empirical results further suggests that our theoretical findings hold for a more realistic settings.

To the present date, mean-field theories have not provided the explicit expression like ridgelet transform because they consider the integral- representation without ridgelet transform. 
If we know that the local minima tends to a ridgelet spectrum, then we can further understand the theoretical backgrounds behind the \emph{lazy learning}, a recent trend of inductive bias theories, such as the \emph{neural tangent kernel} \citep{Jacot2018,Lee2019} and the \emph{strong lottery ticket hypothesis} \citep{Frankle2019}, claiming that the learned parameters are very close to the initial parameters. This is reasonable when the initial parameters cover the support of the ridgelet spectrum.
As a consequence, this study develops a new direction of the theoretical studies of local minima. See Related Works (Section \ref{sec: related works}) for more discussions.

Contributions are summarized as follows: This study
\begin{itemize}
    \item develops a complete set of the ridgelet transform on the torus including reconstruction formula, admissible condition, Plancherel formula, boundedness, density, and several formulas for calculus;
    \item mathematically proves (1) that the population risk minimizer of the ridge regression problem with integral representation NNs is expressed by the ridgelet transform, and (2) that the empirical risk minimizer (ERMer) of the ridge regression problem with finite two-layer NNs converges to the ridgelet transform in the over-parametrized regime, namely, when the parameter number and the sample size tend to infinity;
    \item empirically confirms that the parameter distributions in finite two-layer NNs trained by stochastic gradient descent (SGD) visually converge to the ridgelet spectrum obtained by numerical integration; and
    \item develops a new direction of the theoretical studies of local minima that would reinforce a wide range of recent global convergence theories including mean-field theories and lazy learning.
\end{itemize}
The structure of this paper is as follows: 
In Section 2, we develop the theory of the ridgelet transform on the torus. 
In Section 3, we give our main results. 
In Section 4, we conduct numerical simulation. 
In Sections 5, we discuss the relation to previous studies.
In Section 6, we provides conclusions and further discussions.
\paragraph{Notations.}
The $m$ is the dimension of the Euclidean space of the input data.
We denote by $\dd\xx$ the Lebesgue measure on $\RR^m$.

We denote by $\TT$ the torus $\RR / T \ZZ$ for a fixed $T>0$, which is identified with the interval $[-T/2,T/2)$.
We denote by $\dd b$ the invariant measure on $\TT$, that is identical with the Lebesgue measure on $[-T/2, T/2)$ via the above identification.

For $A >0$, we denote by $\II_A$ the interval $[-A,A]$.
We denote by $\dd\aa$ the Lebesgue masure on $\II_A^m$.
We define $\mu_A:= \dd\aa\dd{b}$ a measure on $\II_A^m\times\TT$.

For a measurable space $X$ equipped with a measure $\mu$, we denote by $L^p(X,\mu)$ the space of $L^p$ integrable functions on $X$ with respect to $\mu$. 
For simplicity, we write $L^p(\mu)$ if $X$ is obvious in context, or write $L^p(X)$ when $\mu$ is the Lebesgue measure or the invariant measure on $\TT$.

For a topological space $X$, we denote by $C_b(X)$ the Banach space of bounded continuous functions on $X$ equipped with the uniform norm.

For a periodic function $\sigma:\TT\to\RR$ and an integer $n$, we write the Fourier coefficient as $\widehat{\sigma}(n):=(1/T)\int_{-T/2}^{T/2}\sigma(t)e^{2\pi int/T}\dd t$. 

For a function $\sigma : \RR \to \RR$, 
$\sigma_{\aa,b}$ denotes a function $\xx \mapsto \sigma(\aa \cdot \xx - b)$ 
, and $\sigma_{\xx}$ denotes a function $(\aa,b) \mapsto \sigma(\aa \cdot \xx - b)$.

\section{RIDGELET TRANSFORM ON THE TORUS}
In this section, we establish the theory of the ridgelet transform on the torus, which is a basis of this study.
For those who are not familiar with ridgelet analysis, we refer to the \emph{Cheat Sheet} (Appendix \ref{sec:cheat}) including the list of handy formulas and the visualization of reconstruction formula with admissible and non-admissible functions.

The ridgelet transform on the torus is a complete set of new ridgelet transform,
because periodic activation functions cannot be \emph{self-admissible} in the non-periodic context, and thus two theories are exclusive to each other.
We need the self-admissibility for the Plancherel formula to hold.

\subsection{Periodic Activation Function}
In this study, we consider the activation function $\sigma$ to be bounded and measurable function from $\TT$ to $\RR$,
or equivalently, a bounded measurable periodic function $\sigma$ on $\RR$ with period $T$: $\sigma(t+T)=\sigma(t)$.

Originally, the ridgelet transform is defined on the real line $\RR$ \citep{Murata1996,Candes.PhD}.
However, the non-compactness of $\RR$ gives rise to several technical difficulties in the proofs, especially, in establishing a connection between the ridgelet transform and finite neural networks.
Moreover, the original definition excludes non-integrable activation functions such as the hyperbolic tangent function and the rectified linear unit (ReLU).
\citet{sonoda2015} have extended the ridgelet transform to accept such non-integrable activation functions,  by introducing an auxiliary dual activation function.
However, their extension sacrifices the Plancherel formula, which we need in this study.

Although it might be possible to develop a truncated version of the ridgelet theory, such as the ``ridgelet transform on a closed interval'', it disables us from using fruitful results in Fourier analysis.
In contrast, if we impose a periodicity on $\sigma$, we can use a quite powerful mathematical machinery, that is the theory of the Fourier transform on the torus $\TT \simeq [-T/2,T/2)$. 
Since we may take arbitrarily large $T$, it is not so harmful as we often consider a finite dataset that is always contained in a (sufficiently large) compact domain.
It is worth remarking that there exists a study \citep{Sitzmann2020} that utilizes a periodicity of the activations, in which the authors report neural networks with periodic activations perform better in some machine learning tasks using real world data.

\subsection{Integral Representation of Neural Networks}
\label{sec: integral representation}
We give a definition of an integral representation.  
\begin{dfn}[Integral Representation]
\label{def: integral representation}
Let $\sigma: \TT \rightarrow \RR$ be a bounded measurable function, and let $P$ be a finite Borel measure on $\RR^m$. For any finite Borel measure $\lambda$ on $\RR^m\times \TT$, we define an integral representation of a neural network $S_\lambda: L^2(\lambda)\rightarrow \Fp $ by
\begin{align}
S_\lambda[\gamma](\xx):=\int_{\RR^m\times\TT} \gamma(\aa,b) \sigma(\aa\cdot\xx-b) \dd\lambda(\aa, b).
\end{align}
\end{dfn}
In this study, we mainly consider two cases: $\lambda_d=\sum_{i=1}^d\delta_{\aa_i,b_i}$, and $\lambda=\mu_A$.
As for the first case, $S_{\lambda_d}[\gamma](\xx) = \sum_{i=1}^d\gamma(\aa_i,b_i)\sigma(\aa_i\cdot\xx -b_i)$.
Thus $S_{\lambda_d}$ represents a finite two-layer neural network.
As for the second case, the operator $S_{\mu_A}$ can be regarded as a continuum limit of neural networks whose hidden parameters $(\aa_i,b_i)$ are contained in $\II_A^m\times\TT$.

Here, we provide a remark on the space $L^2(P)$. As $\Fx$ does not contain $\sigma_{\aa,b}$ and thus any finite neural networks, we cannot see the direct connection between finite neural networks and integral representations of neural networks in $\Fx$.
To circumvent this technical issue, we consider $L^2(P)$ since $\sigma_{\aa,b} \in \Fp$.

We note the boundedness of the integral representation:
\begin{prop}
\label{prop: boundedness of S}
The linear operator $S_{\lambda}$ is bounded, namely, there exists a positive constant $C>0$ such that $\| S_\lambda[\gamma] \|_{L^2(P)} \le C\| \gamma \|_{L^2(\lambda)}$ for all $\gamma \in L^2(\lambda)$.
\end{prop}
The boundedness is a sufficient condition to establish the unique existence of the global optimum in the learning problem (\ref{square loss tikhonov regularization}) we will consider.

\subsection{Ridgelet Transform}
Let us introduce an assumption on the bounded measurable function $\sigma$.
\begin{asm}[Admissible Condition]
\label{asm: admissible condition}
The function $\sigma: \TT \rightarrow \RR$ is bounded and measurable, and satisfies the following two conditions: (1) $\widehat{\sigma}(0)=0$, and (2) $T^{m+1}\sum_{n\neq0}|\widehat{\sigma}(n)|^2/|n|^m=1$.
\end{asm}
We need the admissibility condition (AC) in the proof of the reconstruction formula (\ref{inversion formula}) below. 
It is not at all strong. In fact, the infinite sum in the second condition always converge because $\sigma$ is square integrable, thus, 
we may replace $\sigma$ with a function satisfying these condition via only multiplying and subtracting constants.
For example, restrictions of ReLU and hyperbolic tangent to $\TT$ with slight modifications on the constants, namely, ${\rm 
ReLU}|_{[-T/2,T/2]} - T/8$ and $\tanh|_{[-T/2,T/2]}$ are admissible.
Note that we can further eliminate the constant $-T/8$ in the ReLU by simply adding an offset $b_0$ to the model as $S[\gamma] + b_0$. It is a routine to extend our analysis for $S[\gamma]+b_0$ to have a parallel consequence. In this case, we do not need Item 1 of the AC.

We introduce the \emph{ridgelet transform} and its reconstruction formula. %
\begin{dfn}[Ridgelet Transform]
\label{def: ridgelet transform}
Impose Assumption \ref{asm: admissible condition} on $\sigma$. Then,
we define the ridgelet transform $R:\Fx\rightarrow \Grt$ by
\begin{align}
R[f](\aa,b):=\int_{\RR^m}f(\xx)\sigma(\aa\cdot \xx-b)\dd{\xx}.
\end{align}
\end{dfn}
For  a rigorous treatment of the well-definedness of the ridgelet transform, see Remark \ref{rmk: well-def of ridgelet} below.
\begin{thm}[Reconstruction Formula]
\label{thm: plancherel inversion}
Impose Assumption \ref{asm: admissible condition} on $\sigma$. Then for $f, g\in \Fx$, we have
\begin{align}
    \lim_{A\rightarrow\infty} \Sa\left[R[f]\right] &= f,\label{inversion formula}\\
    \big\langle R[f], R[g]\big\rangle_{\Grt} &=\langle f, g\rangle_{\Fx}.\label{plancherel formula}
\end{align}
\end{thm}

By discretizing the integral in (\ref{inversion formula}), we have a stronger result of a well-known universality of two-layer neural networks as a corollary of Theorem \ref{thm: plancherel inversion}:
\begin{cor}
\label{cor: universality}
Impose Assumption \ref{asm: admissible condition} on $\sigma$, and assume $f$ is a rapidly decreasing smooth function. 
Then,
for an arbitrary $\varepsilon>0$ and a compact domain $K \subset \RR^m$, there exists $A>0$ and $d>0$ such that the following inequality almost surely holds:
\[\left\|\frac{(2A)^mT}{d}\sum_{i=1}^d R[f](\aa_i,b_i)\sigma_{\aa_i,b_i}-f\right\|_{L^\infty(K)}<\varepsilon,\]
where $(\aa_i,b_i)$'s are i.i.d samples drawn from the uniform distribution over $\II_A^m\times\TT$.

\end{cor}
Typical universality results only concern approximation power of neural networks.
Such results guarantee the representation power of neural networks, however, their parameters could become too large to be realized in the real world.
In contrast, Corollary \ref{cor: universality} provides us not only the approximation power but also detailed information of the parameter distributions.
Although there might be many candidates of neural networks that represent the target function, Corollary \ref{cor: universality} shows that one of them are given by the ridgelet transform, a simple integral transform.
Conversely, under over-parametrized condition, we will prove that the parameter distribution of an optimal neural network is closely related to the ridgelet transform.
\begin{rmk}
\label{rmk: well-def of ridgelet}
For mathematical and logical accuracy,  we need to define $R[f]$ for all the $f \in \Fx$ with Theorem \ref{thm: plancherel inversion} via bounded extension, essentially the same arguments in the definition of the $L^2$-Fourier transform on the Eulidean space.
More precisely, We first define $R[f]$ for $f \in L^1(\RR^m)$, which is absolutely convergent because $\sigma \in L^\infty(\TT)$.
Then, we show the Plancherel formula for $f \in L^1(\RR^m) \cap L^2(\RR^m)$ as in \refthm{ plancherel inversion}.
Finally, we extend $R[f]$ for $f \in \Fx$ as a common limit of $R[f_i]$, where $f_i$ is any sequence in $L^1(\RR^m)\cap \Fx$ that converges to $f$ in $L^2(\RR^m)$. 
\end{rmk}

\section{MAIN RESULTS}
In this section, we describe the formulation of our problem and main results (Theorems \ref{thm: minimizer = ridgelet} and \ref{thm: convergence of finite NN}). 
We fix an activation function $\sigma: \TT\rightarrow\RR$. We assume that $\sigma$ is continuous almost everywhere, equivalently, Riemann integrable, and satisfies Assumption \ref{asm: admissible condition}.
We also fix a square integrable function $f\in L^2(\RR^m)$ as a data generating function, and an absolutely continuous probability measure $P$ on $\RR^m$ with bounded density function $p\in L^1(\RR^m)$ as the input data distribution.
We write an empirical measure corresponding to $P$ by $P_N:=1/N\sum_{i=1}^N \delta_{\xx_i}$, where $\xx_1,\dots,\xx_N \in \RR^m$ are i.i.d samples drawn from $P$.
For $A,T>0$ and $d \in \NN$, let 
\begin{align}
\Lambda_{d,A} := \left\{ \frac{C_0}{d} \sum_{i=1}^d \delta_{(\aa_i,b_i)} \mid (\aa_i,b_i) \in \II_A^m \times \TT \right\},
\end{align}
where $C_0 := (2A)^mT$,
be the collection of $d$-term \emph{hidden parameter distributions} on $\II_A^m \times \TT$.

\paragraph{Main Claim (Theorems \ref{thm: minimizer = ridgelet} and \ref{thm: convergence of finite NN})}
Our main results are summarized in the following formula, which is a converse of Corollary \ref{cor: universality}:
\begin{align}
     \lim_{N\rightarrow\infty} \lim_{d\rightarrow\infty}\gamma^*_{N,d} = R\left[\frac{pf}{\beta+p}\right] + \Delta_A, \notag
\end{align}
where $\gamma^*_{N,d}$ represents the parameter distribution of $d$-term two-layer neural networks trained by regularized empirical risk minimization with $N$ training examples, $\beta$ is a regularization parameter, and $\Delta_A$ is a small residual term that tends to $0$ as $A \to \infty$. %

We call this inner limit $\lim_{d\rightarrow\infty}\gamma^*_{N,d}$ along the parameter number $d$ the \emph{over-parametrization}.
In this sense, we show ``over-parametrized networks converge to the ridgelet transform.''

\subsection{Square Loss Minimization}

For an arbitrary $\beta>0$, any finite Borel measures $\lambda$ and $P$ on $\RR^m \times \TT$ and $\RR^m$, respectively, and any square integrable function $f \in L^2(P)$,
we consider the following form of $L^2$-regularized square risk: %
\begin{align}
\loss( \gamma ; f, P, \lambda,\beta):=\left\Vert f- S_\lambda[\gamma]\right\Vert^2_{L^2(P)}
+\beta \left\Vert \gamma \right\Vert^2_{L^2(\lambda)}. \label{eq:square risk}
\end{align}
We denote by $\gstar[f;P,\lambda,\beta]$,
the unique element that attains the minimum
\begin{align}
\label{square loss tikhonov regularization}
    \min_{\gamma \in L^2(\lambda)} \loss(\gamma; f,P,\lambda,\beta),
\end{align}
which always exists as long as $S_\lambda$ is a densely defined closed operator (see Appendix \ref{app:tikhonov} for the proof).
As we have already seen in Proposition~\ref{prop: boundedness of S}, $S_\lambda$ is bounded, and thus the minimizer always exists.

The minimizer of (\ref{square loss tikhonov regularization}) behaves well under limit manipulations, namely, we have the following lemma:
\begin{lem}
\label{lem: limit of minimizers}
For any finite Borel measure $\lambda$ on $\RR^m\times\TT$, as $N\rightarrow\infty$, we have
\[\left\|\gamma^*[f; P_N, \lambda, \beta] -\gamma^*[f; P, \lambda, \beta] \right\|_{L^2(\lambda)} \to 0\text{~~$P$-a.s.}\]
\end{lem}

Now we consider two types of minimization problems.
Our goal is to describe the relationship between the minimizers as well as investigate the properties of them.

\paragraph{Continuous Population Risk Minimizer $\gamma^*$.}
We denote by $\gamma^* $ the population risk minimizer of %
\begin{align}
     & \min_{\gamma\in L^2(\mu_A) }J(\gamma; f, P, \mu_A, \beta).\label{min prob integral representation}
\end{align}
The minimizer $\gamma^*$ is equal to $\gamma^*[f;P,\mu_A,\beta]$, and referenced as a theoretically ideal object, which shows up as the global minimizer with over-parametrized neural networks.

\paragraph{Finite Empirical Risk Minimizer $\gamma^*_{N,d}$.}
We denote by $\gstar_{N,d}$ an empirical risk minimizer of
\begin{align}
     & \min_{\lambda \in \Lambda_{d,A}} \min_{\gamma\in L^2(\lambda)} J(\gamma ;f, P_N, \lambda, \beta_d).\label{min prob empirical}
\end{align}
By definition, the minimization problem (\ref{min prob empirical}) is equivalent to an ordinary learning problem of two-layer neural networks in term of the following empirical risk with respect to the parameters $(\aa_j,b_j,c_j)\in \II^m_A \times \TT \times \RR$:
\begin{align}
    \frac{1}{N} \sum_{i=1}^N \left| f(\xx_i)-\frac{C_0}{d}\sum_{j=1}^d c_j\sigma_{\aa_j,b_j}(\xx_i)\right|^2 + \beta_d \frac{C_0}{d}\sum_{j=1}^d|c_j|^2, \label{explicit loss of finite NN}
\end{align} 
where we write $C_0:=(2A)^mT$.
By definition, $\gamma_{N,d}$ attains the minumum of (\ref{square loss tikhonov regularization})  for some $\lambda_d^* \in \Lambda_{d, A}$, namely, we have $\gamma^*_{N,d}=\gamma^*[f;P_N, \lambda^*_d,\beta_d]$. We call $\lambda$ the \emph{hidden parameter distribution} of the ERMer $\gamma^*_{N,d}$.

As we see soon later, our main theorem holds not only for the strict global minimizer but also for more general solutions that satisfy a very mild assumption: %
\begin{asm}\label{asm: support}
A sequence of hidden parameter distributions $\{ \lambda_d \}_{d=1}^\infty$ ($\lambda_d\in\Lambda_{d,A}$) weakly converges to the uniform distribution $\mu_A$ over the parameter domain $\II_A^m\times\TT$,
namely, for any bounded continuous function $h \in C_b(\II_A^m\times\TT)$,
$\int h \dd \lambda_d \to \int h \dd \aa \dd b$ as $d \to \infty$.
\end{asm}

Here, we remark for potential confusions: The objective function \eqref{square loss tikhonov regularization} may remind some readers of the kernel ridge regression (KRR) with either $k(\xx,\yy) := \EE_{\aa,b \sim \lambda}[ \sigma( \aa \cdot \xx - b ) \sigma( \aa' \cdot \xx - b' )  ]$ on the data space, or $K((\aa,b),(\aa',b')) := \EE_{\xx \sim P}[ \sigma( \aa \cdot \xx - b ) \sigma( \aa' \cdot \xx - b' )  ]$ on the parameter space.
However, both KRRs cannot deal with our problem \eqref{square loss tikhonov regularization}.
Recall that our final goals are to specify the parameter distribution $\gamma^* \in \Git$ and to show the convergence of the finite minimizers $\gamma_d = \sum_{i=1}^d c_i \delta_{(a_i,b_i)}$ to $\gamma^*$. 
In general, $\gamma^*$ involves null component when $\beta>0$, but $H_K$ does not involve null components and thus the minimizer $\gamma^*$ in $L^2(\mu_A)$ cannot always included in $H_K$.

\subsection{Explicit Representation of Continuous Minimizer}
The first main result is the explicit representation of the continuous minimizer $\gamma^*$, the solution of (\ref{min prob integral representation}), in terms of the ridgelet transform. 
\begin{thm}
\label{thm: minimizer = ridgelet}
Let $f\in L^2(\RR^m)$ be a bounded square integrable function, and let $P$ be an absolutely continuous probability measure on $\RR^m$ with bounded density function $p\in L^1(\RR^m)$.
For $A>0$ and $\beta>0$, we have
\begin{align}
    \gstar=R\left[\frac{pf}{\beta+p}\right]+\Delta_{A},
\end{align}
where $\Delta_{A}$ is an element of $\Git$ such that 
\begin{align}
\lim_{A\rightarrow\infty}\left\Vert \Delta_{A}\right\Vert_{\Git}=0.
\end{align}
\end{thm}
By Corollary \ref{cor: universality}, it is reasonable to expect the minimizer $\gamma^*$ and the ridgelet transform are intimately related to each other.
However, since there exists a nonzero element $\gamma_0\in L^2(\mu_A)$ satifying $\Sa[\gamma_0]=0$, $R[f]+\gamma_0$ also provides a parameter distribution that approximates the target $f$ well.
Theorem \ref{thm: minimizer = ridgelet} shows that the regularization term removes the effect of $\gamma_0$, and the minimizer $\gamma^*$ coincodes with the ridgelet transform except a small oscillation $\Delta_A$.

The principal term $\gamma_{pri}^* := R[fp/(\beta+p)]$ of the obtained minimizer is understood as a \emph{shrinkage estimator}, or a \emph{biased estimator}, of $\gamma^\circ:=R[f]$.
Namely, while $\gamma^\circ$ exactly attains $S[\gamma^\circ]=f$, the obtained estimate $S[\gamma_{pri}^*]=fp/(\beta+p)$ is intentionally biased from $f$, and the norm $\| \gamma_{pri}^*\|$ is intentionally $p/(\beta +p)$-times smaller than $\|\gamma^\circ \|$. Recall that a regularized estimator is generally a biased estimator, and shrinkage is a natural consequence of ridge regression because the regularizer $\beta \| \gamma \|^2$ penalizes the norm of $\gamma$. 

As described in Proposition \ref{limit projection} for general settings, if $\beta \to +0$, then $\gamma^*$ converges to the \emph{minimum norm solution}. In our setting, by the continuity in $\beta$, it is simply given by $\lim_{\beta \to +0} \gamma_{pri}^* = \lim_{\beta \to 0} R[fp/(\beta + p)] = R[f] = \gamma^\circ$. However, we remark that this \emph{does not mean} that ``If we minimize \refeq{square risk} without any regularization (by letting exactly $\beta=0$), then $\gamma^*=R[f]$''. In this case, the correct answer is $\gamma^*=R[f] + \ker S$. Namely, the minimizer will have a redundancy in null space $\ker S$.

\subsection{Convergence of Finite Minimizers in the Over-parametrization Regime}
The second main result is a convergence of parameter distributions of finite neural networks with over-parametrization.
\begin{thm}
\label{thm: convergence of finite NN}
Let $\{ \gamma^*_{N,d} \}_{d=1}^\infty$ be a sequence of ERMers. Impose Assumption \ref{asm: support} on the hidden parameter distributions $\lambda_d$ of $\gamma^*_{N,d}$, namely, $\lambda_d$ weakly converges to $\mu_A$. Assume $\beta_d\rightarrow \beta$ as $d\rightarrow\infty$.
Then, for any bounded continuous function $h$ on $\II_A^m\times\TT$, we have
\begin{align}
\lim_{N\rightarrow\infty}\lim_{d\rightarrow\infty}\int h \gamma^*_{N,\lambda_d}\dd{\lambda_d} = \int h \gamma^* \dd\mu_A.    
\end{align}
Here the limit with respect to $N$ is in the sense of $P$-a.s. convergence.
\end{thm}

Theorem \ref{thm: convergence of finite NN} claims that the over-parametrized two-layer neural networks weakly converges to the population risk minimizers $\gamma^*$ as the sample size gets increased.
Combined with Theorem \ref{thm: minimizer = ridgelet}, we obtain the statement ``an over-parametrized neural network converges to the ridgelet spectrum''.
The ``weak convergence'' is not much weak because if we take an arbitrary region of interest $K \subset \RR^m$, and let 
the indicator function $1_K$ to be the test function $h$, then the parameter distribution eventually converges to $\int_K R[f](\aa,b) \dd \aa \dd b$.
In Section \ref{sec:motivation} below, %
we will see that the parameters of finite neural networks trained by SGD accumulate the ridgelet spectrum.

We provide a remark on the assumption that $\lambda_d$ converges to a measure $\lambda$.
Since we consider the ridge regression, the support of ERMers cannot concentrate in a null set, for example, a lower dimensional submanifold, as the parameter number gets increased.
More precisely, we have the following simple lemma:
\begin{lem}
\label{lem: support collapse}
Let $\lambda$ be a finite Borel measure on $\RR^m\times\TT$.
Let $\gamma \in L^2(\lambda)$. Assume $\| \gamma \|_{L^1(\lambda)} >C$ for some $C>0$.
Then we have $\| \gamma \|_{L^2(\lambda)} > C/\lambda({\rm supp}(\gamma))$.
\end{lem}
This lemma implies if the the support of coefficient functions is collapsed to a null set, then its $L^2$-norm explodes. 
Therefore, the ridge regularization exclude such a coefficient function as a solution of the minimization problem in question.

\begin{proof}[Proof of Theorem \ref{thm: convergence of finite NN}]
We provide a sketch of the proof.
In fact, we prove a stronger convergence result as follows:
\begin{lem}
\label{lem: convergence with over parametrization}
Let $\{\lambda_d\}_{d=1}^\infty$ ($\lambda_d \in \Lambda_{d,A}$) be a sequence of a finite Borel measure. Impose Assumption \ref{asm: support} on $\{ \lambda_d \}_{d=1}^\infty$.
Assume that $\beta_d\rightarrow\beta$ as $d\rightarrow\infty$.
Then, as $d\rightarrow \infty$, we have
\[\big\| \gamma^*[f;P_N,\lambda_d,\beta_d] - \gamma^*[f; P_N, \mu_A,\beta] \big\|_{L^2(\gamma_d)} \longrightarrow 0.\]
\end{lem}
As a consequence, in the over-parametrized regime, the convergence occurs even when the hidden parameters are not optimized but at least they converge to $\mu_A$.
Combined with Lemma \ref{lem: limit of minimizers}, the minimizer $\gamma^*[P_N, \mu_A]$ almost surely converges to $\gamma^*$ as $N\rightarrow\infty$.
\end{proof}

\section{NUMERICAL SIMULATION}\label{sec:motivation}

In order to verify the main results, we conducted numerical simulation with artificial datasets.
Here, we only display the results of Experiment~1.
The readers are also encouraged to refer Appendix~\ref{app:experiments} for further experimental results.

\begin{figure*}
    \centering
    \begin{subfigure}[b]{0.3\textwidth}
            \includegraphics[width=\linewidth]{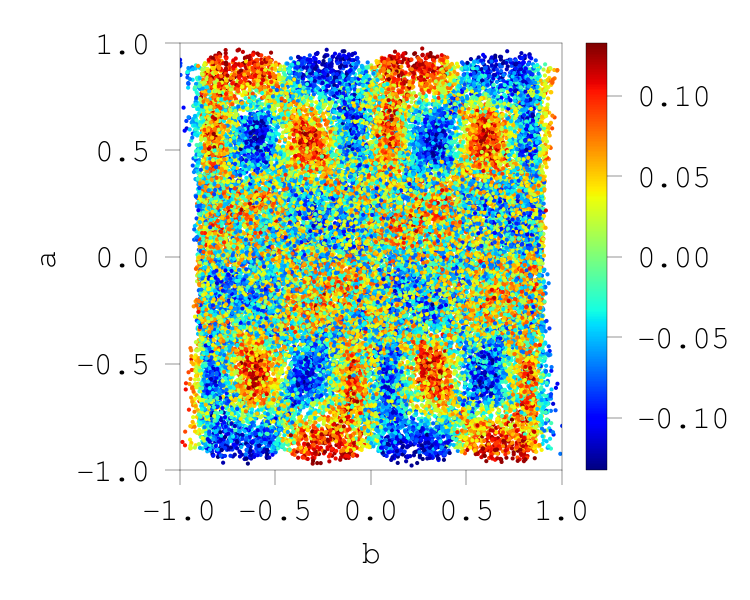}
            \caption{SGD params, Gaussian}
            \label{fig:dist gauss}
    \end{subfigure}%
    \begin{subfigure}[b]{0.3\textwidth}
            \includegraphics[width=\linewidth]{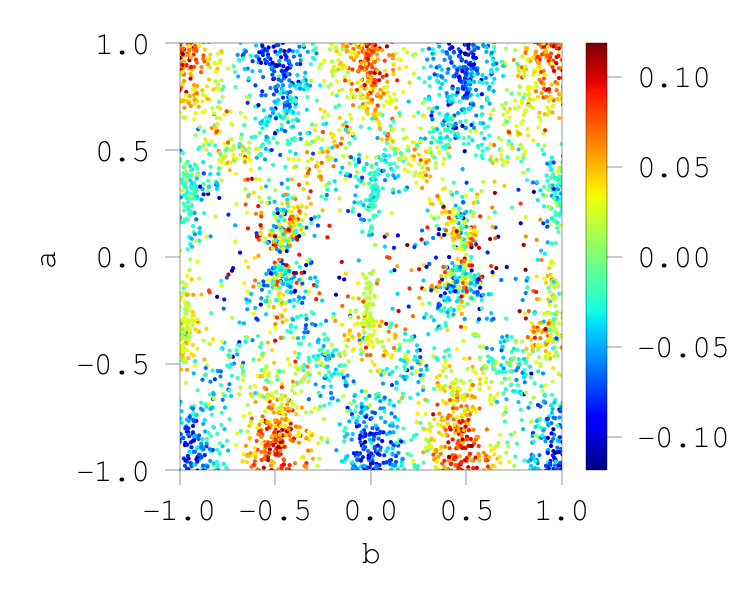}
            \caption{SGD params, Tanh}
            \label{fig:dist tanh}
    \end{subfigure}%
    \begin{subfigure}[b]{0.3\textwidth}
            \includegraphics[width=\linewidth]{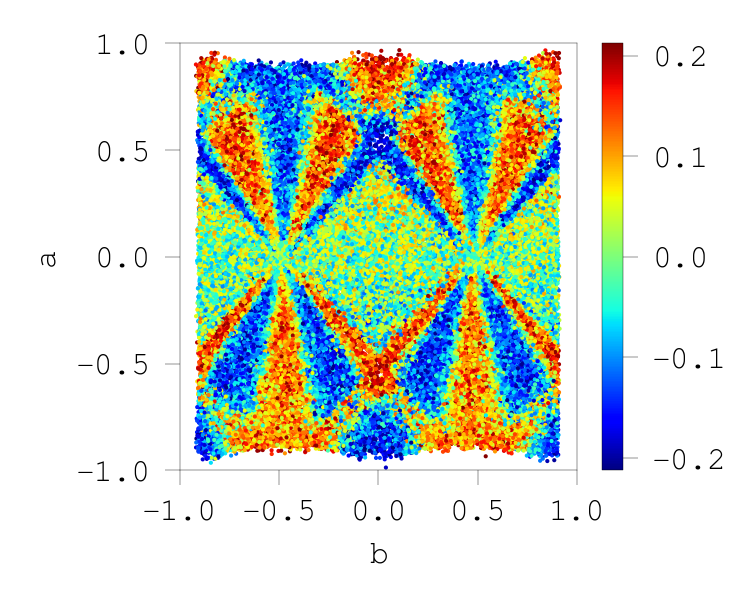}
            \caption{SGD params, ReLU}
            \label{fig:dist relu}
    \end{subfigure}\\
    \begin{subfigure}[b]{0.3\textwidth}
            \includegraphics[width=\linewidth]{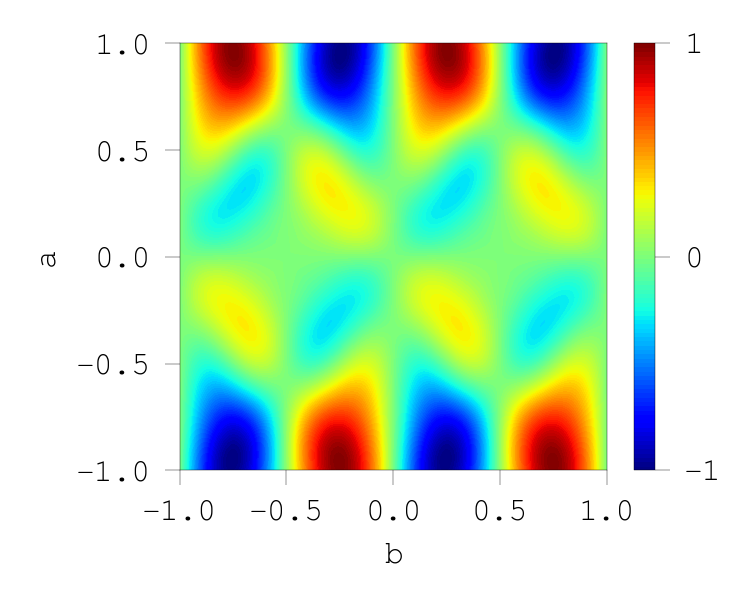}
            \caption{R. spect, Gaussian}
            \label{fig:spect gauss}
    \end{subfigure}%
    \begin{subfigure}[b]{0.3\textwidth}
            \includegraphics[width=\linewidth]{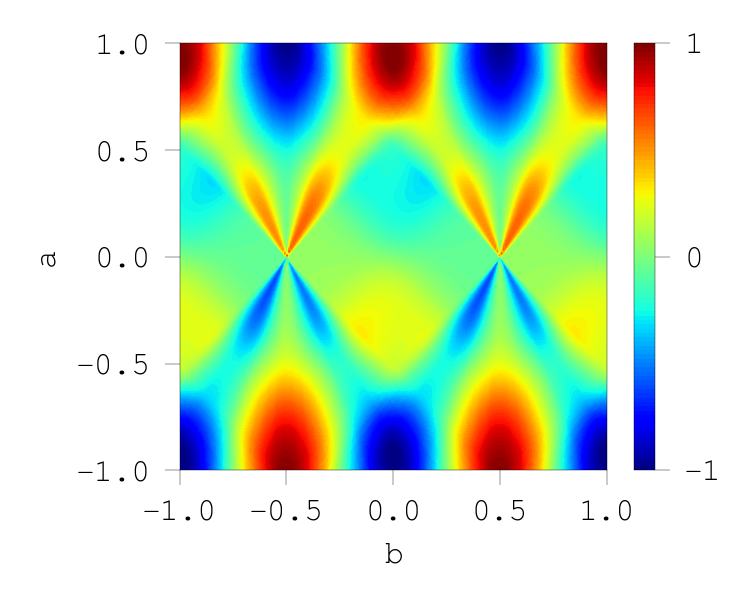}
            \caption{R. spect, Tanh}
            \label{fig:spect tanh}
    \end{subfigure}%
    \begin{subfigure}[b]{0.3\textwidth}
            \includegraphics[width=\linewidth]{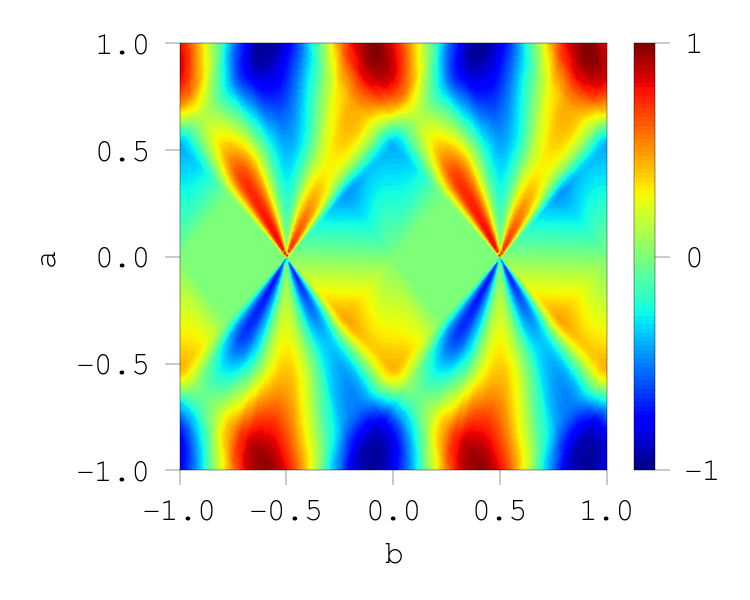}
            \caption{R. spect, ReLU}
            \label{fig:spect relu}
    \end{subfigure}%
    \caption{Parameter distributions $\gamma(\aa,b)$ trained by SGD (top) and ridgelet spectra $R[f](\aa,b)$ obtained by numerical integration (bottom) for the common data generating function $f(x) = \sin 2 \pi x, (x \in [-1,1])$.}\label{fig:ridgelet}
\end{figure*}

\subsection{Data Generation}
For the sake of visualization, all the datasets are $1$-in-$1$-out, so that the scatter plot will be displayed in a three-dimensional manner: $(a,b) \in \RR^2$ in position and $c \in \RR$ in color.
However, we remark that our theoretical results are valid for any dimension.
We always consider the uniform distribution $x_i \sim U(-1,1)$ for the input vectors,
and generate $n = 1,000$ samples for training, except for the case of \emph{Topologist's Sine Curve (TSC)} $y_i = \sin \frac{2 \pi}{x_i}$.
For the TSC, we generate $n = 10,000$ because the frequency tends to infinity as $x$ tends to 0.

\subsection{Scatter Plot of SGD Trained Parameters}

Given a dataset $D_n = \{ (x_i, y_i) \}_{i=1}^n$, we repeatedly train $s=1,000$ neural networks $g(x ; \theta^{(t)}) = \sum_{i=1}^d c_i^{(t)} \sigma( a_i^{(t)} x - b_i^{(t)} ), (t \in [s])$ with activation function $\sigma = $ periodic Gaussian, periodic Tanh and periodic ReLU.
The training is conducted by minimizing the square loss: $L(\theta) = \frac{1}{n} \sum_{i=1}^n | y_i - g( x_i ; \theta ) |^2$ using stochastic gradient descent (SGD) with learning rate $\eta >0$ and weight decay rate $\beta > 0$.
Note that the weight decay has an equivalent effect to the $L^2$-regularization. In the main theory, only $c$ is imposed $L^2$-regularization, and $(a,b)$ are strictly restricted in a compact domain $\II_A^m \times \TT$. However, in the experiments, all the parameters are imposed $L^2$-regularization for the sake of simplicity.
The initial parameters are drawn from the uniform distribution $U(-1,1)$.
All the parameters are updated by SGD, so that this is \emph{not} a random features method \citep{Rahimi2008} in which hidden parameters $(a,b)$ are frozen after initialization.
After the training, we obtain $sd$ sets of parameters $\{ (a_i^{(t)}, b_i^{(t)}, c_i^{(t)}) \}_{t \in [s], i \in [d]}$, and plot them in the $(a,b,c)$-space. ($c$ is visualized in color.)

\subsection{Heatmap of Ridgelet Spectrum}

Given a dataset $D_n = \{ (x_i, y_i) \}_{i=1}^{n}$, we approximately compute the ridgelet spectrum $R[f](a,b)$ of $f$ at every sample points $(a, b)$ by numerical integration:
\begin{align}
R[f](a,b) \approx \frac{1}{n} \sum_{i=1}^{n} y_i \sigma( a x_i - b ) \Delta x,
\end{align}
where $\Delta x$ is a normalizing constant, which is a constant because we assume that $x_i$ be uniformly distributed.
We remark that more sophisticated methods for the numerical computation of the ridgelet transform have been developed.
See \citet{Do2003} and \citet{Sonoda2014} for example.

\subsection{Results}

In \reffig{ridgelet}, we compare the scatter plots of SGD trained parameters and the heatmaps of ridgelet spectra. All six figures are obtained from the common data generating function $f(x) = \sin 2 \pi x$ on $[-1,1]$.
Despite the fact that the scatter plots and heatmaps are obtained from different procedures: numerical optimization and numerical integration, both figures share characteristics in common. For example, red and blue parameters in the scatter plots (a-c) concentrate in the area where the heatmaps (d-f) indicate the same colors.
Due to the periodic assumption, the ridgelet spectrum spreads infinitely in $b$ with period $T=1$. On the other hand, due to the weight decay and initial locations of parameters, the SGD trained parameters gather around the origin. Here, we used the uniform distribution $U(-1,1)$ for the initialization. We can understand that these differences between the scatter plot and ridgelet spectrum as the residual term $\Delta_{A,\beta}$ in the main theorem.
Another remarkable fact is that the SGD trained parameters essentially did not change their positions in $(a,b)$ from the initialized value. This is reasonable when the support of initial parameters overlap the ridgelet spectrum from the beginning. We can understand this phenomenon as the so-called lazy regime.

\section{RELATED WORKS}
\label{sec: related works}

A preprint by \citet{Sonoda2018a} is the closest result with non-periodic $\sigma$. Compared to their result, we improved a lot. In their work, the function class of the data generating gunctions $f$ remains to be an abstract RKHS $H^{\sigma \rho}$, the minimizer $\gstar$ is given as an abstract projection of $R[f]$ onto a closed subspace, a hyper-parameter $\rho$ in the ridgelet transform $R$ remains not specified, and neither finite models nor finite samples are discussed. As far as we have noticed, we are the first to have revealed that the finite empirical minimizers do converge to the ridgelet spectrum.

\paragraph{Earlier Global Convergence Results.}
In the past, many authors have investigated the local minima of deep learning.
However, these results have often posed strong assumptions such as that (A1) the activation function is limited to linear or ReLUs \citep{Kawaguchi2016,Soudry2016,Nguyen2017,Hardt2017,Lu2017, Yun2018}; (A2) the parameters are random \citep{Choromanska2015a,Poole2016,Pennington2018,Jacot2018,Lee2019,Frankle2019}; (A3) the input is subject to normal distribution \citep{Brutzkus2017}; or (A4) the target functions are low-degree polynomials or another sparse neural network \citep{Yehudai2019,Ghorbani2019}.
Due to these simplifying assumptions, we know very little about the minimizers themselves.
In this study, from the perspective of harmonic analysis, we present a stronger characterization of the distribution of parameters in the over-parametrized regime. As a result, our theory 
(A1') accepts a wide range of activation functions,
(A2') need not assume the randomness of parameter distributions,
(A3') need not specify the data distribution,
and (A4') preserves the universal approximation property of neural networks such as the density in $L^2$.

\paragraph{Mean-Field Theory.}
The \emph{mean-field theory} \citep{Rotskoff2018,Mei2018,Sirignano2020,Sirignano2020a} a.k.a. the \emph{gradient flow} theory \citep{Nitanda2017,Chizat2018,Arbel2019} has employed the integral representation and parameter distribution to prove the global convergence.
These lines of studies claim that for the stochastic gradient descent learning of two-layer networks,
the time evolution of a finite parameter distribution, say $\gd(t)$, with parameter number $d$ and continuous training time $t$,
asymptotically converges to the time evolution of the continuous parameter distribution as $d \to \infty$.
Here, the time evolution is described by a gradient flow, called \emph{the partial differential equation, the Wasserstein gradient flow, or the McKean-Vlasov equation}, $\frac{\dd}{\dd t} \gamma_\infty(t) = -\frac{1}{2}\nabla_\gamma \| f - S[\gamma_\infty(t)] \|^2$ with initial condition $\gamma_\infty(0) = \gamma_{init}$.
However, we should point out that these arguments oversights the null component in the parameter distributions. %
As we explained in \refapp{nullS}, the equation $f = S[\gamma]$ has an infinitely different solutions, say $\gamma_1$ and $\gamma_2$ that satisfy $S[\gamma_1] = S[\gamma_2]$ but $\gamma_1 \neq \gamma_2$.
Hence, even though the convergence  $S[\gd] \to S[\gamma_\infty]$ in the function space $L^2(P)$ is established, in general, we \emph{cannot} conclude the convergence  $\gd \to \gamma_\infty$ in the space of parameter distributions $\Git$.
This leaves the parameter distribution indeterminate.
Nevertheless, our numerical simulation results have shown a ``visual'' convergence.
By explicitly posing a regularization term on $\gamma$, we have specified the parameter distribution at the global minimum and have shown that the \emph{weak convergence in the space of parameter distributions}: $\gd \wto R[f]$.
(We remark that some authors consider \emph{noisy SGD}, which is equivalent to imposing the $L^2$-regularization.)

In order to avoid potential confusions, we provide supplementary explanations on the \emph{trick} behind the mean-field theory. In the mean-field theory, the gradient flow $\dd \gamma(t)/\dd t = - \nabla \| f - S[\gamma(t)] \|^2$ is often explained as the \emph{system of interacting particles} by identifying the parameters $\{ (\aa_i, b_i) \}_{i=1}^d$ as the coordinate system of $d$ physical particles. The particles obeys a \emph{non-linear equation of motion} with \emph{interacting potential} $I[\gamma](\aa,b) := \int K(\aa,b; \aa',b') \dd \gamma(\aa',b')$, where $K(\aa,b;\aa',b') := \int \sigma( \aa \cdot \xx - b ) \sigma( \aa' \cdot \xx - b' ) \dd P(\xx)$, which is naturally derived by expanding the square loss function. Based on this physical analogy, this potential seems natural. However, here is the trick because in the potential $I$, 
the null space $\ker S$ is eliminated by \emph{implicitly} applying $S$. Namely, since
\begin{align}
    I[\gamma](\aa,b)
    &= \int \sigma( \aa \cdot \xx - b )  S[\gamma](\xx) dP(\xx),
\end{align}
we can verify that $I[\gamma + \ker S] = I[\gamma]$.
This clearly indicates that the interactive potential is degenerate in $\gamma$,
and thus the mean-field theory would only show a weaker convergence result than our main results.

\paragraph{Lazy Learning.}
The \emph{lazy learning}, such as the \emph{neural tangent kernel} \citep{Jacot2018,Lee2019,Arora2019} and the \emph{strong lottery ticket hypothesis} \citep{Frankle2019}, employs a slightly different formulation of over-parametrization to investigate the inductive bias of deep learning.
These lines of studies draw much attention by radically claiming that the minimizers are very close to the initialized state.
In this study, we revealed that, in the (not lazy but) active regime, the shape of the parameter distribution converges to the ridgelet spectrum.
According to our results, lazy learning is reasonable when the initial parameter distribution covers the ridgelet spectrum in its support, since the initial parameters need not to be \emph{actively} updated. Furthermore, the lazy assumption can be reasonable when the data generating function $f$ is a low frequency function, and thus the ridgelet spectrum $R[f]$ concentrates around the origin, because the initial parameter distribution is typically a normal (or sometimes a uniform) distribution centered at the origin $(\aa,b) = (\bm{0},0)$ and thus eventually the initial parameters cover the ridgelet supectrum.

\paragraph{Implicit Regularization.}
Recently, gradient descent methods are said to impose \emph{implicit regularization} (see eg. \citealp{Zhang2017,Neyshabur.PhD,Gunasekar2018,Gunasekar2018a}),
which often motivates the lazy learning. Although we have no unifying formulation of the implicit regularization to the present, and thus we have simply employed the $L^2$-regularization, %
we may formulate the implicitly regularized problem as the minimization problem of $\loss_{imp}[\gamma;f,\gamma_{init}] := \| f - S[\gamma]\|_{L^2(P)}^2 + \beta \| \gamma - \gamma_{init} \|_{\Git}^2$ for a given initial parameter distribution $\gamma_{init}$ on $\II_A^m\times\TT$.
Then, immediately because $\loss_{imp}[\gamma;f,\gamma_{init}] = \loss[\gamma-\gamma_{init};f]$, we can conclude that the minimizer $\gamma_{imp}^*$ is given by 
$\gamma^* + \proj_{\to \ker S}[\gamma_{init}]$,
as $\beta \to 0$.
Namely, the implicitly regularized solution $\gamma_{imp}$ again meets a ridgelet spectrum $\gamma^*$ but also holds a null component $\proj_{\to \ker S}[\gamma_{init}]$. Investigation of the \emph{role-of-null-space} would be an interesting future work.

\section{CONCLUSION}

In this study, we have derived the unique explicit expression---the ridgelet spectrum with residual---of over-parametrized two-layer neural networks trained by regularized empirical square risk minimization.
To the present, many studies have proven the global convergence of deep learning.
However, we know very little about the minimizer itself because the settings are typically very simplified. %
To investigate the minimizers, we develop the ridgelet transform on the torus, %
which is a complete set of new ridgelet transform. 
The scatter plots of learned parameters have shown a very similar pattern to the ridgelet spectra, which supports our theoretical result. 
Although we considered an idealized ERM, the visual convergence suggested much more. %
Extending our main theorem to a more realistic settings is our important future work. 
Moreover, although we assumed two-layer and ridge regression, as often assumed in recent over-parametrized theories, we conjecture that for a deep network, say $f_2 \circ f_1$ for example, each intermediate layer converges to ridgelet spectrums as $S[R[f_2]]$ and $S[R[f_1]]$; and that for a general loss function $J$, if it is continuous, namely $\| \gamma \| \le C J(\gamma)$, then the minimizer is given as a certain modified version of $R[f]$ (like $R[fp/(\beta+p)]$).

\subsection{Further Discussions after Rebuttal}
The Main Theorems mathematically rigorously show that finite ERMers eventually converge to the \emph{unique closed-form solution} $R[fp/(\beta+p)]$. 
While conventional theories show the global convergence, our theory characterizes the limit point as the ridgelet transform, which complements the conventional theories. 
The uniqueness and the closed-form expression allow us to design theories at a higher resolution than, for example, those that simply assume and/or conclude a sub-Gaussian randomness of parameter distributions. 
For example, we can predict the shape of minimizers as presented in Sections \ref{sec:Rcalc} and \ref{app:experiments}. 
As for the quality of solutions,
by the uniqueness of the minimizer and the continuity of integral representation operator $S$, if the loss value of a current solution $\gamma_{local}$ is $\eps \ge 0$, then the difference vector $\Delta \gamma := \gamma_{local} - \gamma_{global}$ is as small as $O(\eps)$ in $L^2(\RR^m\times\RR)$. Therefore, it is reasonable to say that regardless of the training process, a near-optimal solution also has a similar shape with ridgelet spectrum.

In the mean-field theory, it is known that the parameter distribution converges to a stable distribution, a.k.a. a \emph{Gibbs distribution}, $\gamma_\infty \propto \exp(-\beta L)$ with regularization parameter $\beta$ and loss function $L$, under certain convergence conditions \citep{Mei2018, Tzen2020, Suzuki2020}. The existence of such a distribution is a natural consequence of the fact that SGD is a stochastic gradient flow induced by a locally convex function. Note, however, that the Gibbs distribution contains an unknown loss function $L$, so in general the limit point itself cannot be given explicitly. In other words, the Gibbs distribution is an \emph{equation} that encodes the sufficient conditions for a parameter distribution $\gamma$ to be a limit point. In order to obtain the limit point in closed form, we need to solve this equation.
The ridgelet transform can be understood as a closed-form solution for the Gibbs distribution.
(To be exact, however, 
this study does not fully consider the convergence conditions proposed in mean-field theories, 
simply because these are still developing, and the current version of the convergence conditions are quite restrictive.)
Again, closed-form solutions are more informative than equations.

%% file: aistats2021/supp.tex
\section{CHEAT SHEET FOR RIDGELET TRANSFORM ON $\TT$} \label{sec:cheat}

We identify the torus $\TT := \RR / T\ZZ$ as $[-T/2,T/2)$ some $T>0$. We write $\omega_n := 2 \pi n / T$ for every $n \in \ZZ$.

\subsection{Fourier Transforms and Fourier Expansions}
\paragraph{Fourier Transform on $\TT$, or Fourier Series Expansion.}
Let $T>0$. For any $f \in L^2([-T/2,T/2])$,
\begin{align}
    \widehat{f}(n) &:= \frac{1}{T} \int_{-T/2}^{T/2} f(t) e^{-i \omega_n t} \dd t, \\
    f(t) &= \lim_{N \to \infty} \sum_{n=-N}^N \widehat{f}(n) e^{i\omega_n t}.
\end{align}
In particular, the convolution theorem holds:
\begin{align}
    \widehat{f * g}(n) = T \widehat{f}(n) \widehat{g}(n)
\end{align}

\paragraph{Fourier Transform on $\RR^m$.}
In order to avoid the potential confusion, we write $\sharp$ and $\flat$ for the Fourier transform on $\RR^m$: %
\begin{align}
     f^\sharp(\xxi) &:= \int_{\RR^m} f(\xx ) e^{-i \xx  \cdot \xxi } \dd \xx , \quad \xx  \in \RR^m \\
      f^\flat(\xxi) &:= \frac{1}{(2\pi)^m}\int_{\RR^m} f(\xx ) e^{i \xx  \cdot \xxi } \dd \xx , \quad \xxi \in \RR^m \\
     f(\xx ) &= \frac{1}{(2 \pi)^m} \int_{\RR^m} f^\sharp(\xxi) e^{i \xx  \cdot \xxi} \dd \xxi, \quad \xx  \in \RR^m.
\end{align}

\subsection{Ridgelet Transform}
\label{app: ridglet transform}
Here we introduce a general form of ridgelet transforms (\ref{general ridgelet transfomrs}) in terms of another bounded periodic funcition $\rho$.  In the main body, we use this theory in the case of $\rho=\sigma$. Assumption \ref{asm: admissible condition} corresponds to (\ref{admissible 1}) and (\ref{admissible 2}).
\paragraph{Integral Representation.} Let $\lambda$ be a finite Borel measure on $\RR^m\times\TT$. For any $\gamma \in L^2(\lambda)$ and $\sigma \in L^\infty(\TT)$,
\begin{align}
    S_\lambda[\gamma](\xx) &:= \int_{\RR^m\times\TT} \gamma(\aa,b)\sigma(\aa\cdot\xx-b)\dd\lambda(\aa, b), \quad \xx \in \RR^m
\end{align}

\paragraph{Ridgelet Transform.} For any $f \in L^1(\RR^m)$ and $\rho \in L^\infty(\TT)$,
\begin{align}
    R[f](\aa,b) &:= \int_{\RR^m} f(\xx ) \overline{\rho(\aa\cdot \xx  - b)} \dd \xx , \quad (\aa,b) \in \RR^m \times \TT\label{general ridgelet transfomrs}
\end{align}
If $\rho$ satisfies Assumption \ref{asm: admissible condition} (admissible with itself), then we can extend the ridgelet transform to $f \in L^2(\RR^m)$.

\paragraph{Adjoint Operator.}
For $\gamma\in {\rm Im}(R:L^2(\RR^m)\to L^2(\RR^m\times\TT))$,
\begin{align}
    R^*[\gamma](\xx ) &:= \int_{\RR^m \times \TT} \gamma(\aa,b) \rho( \aa\cdot \xx  - b ) \dd \aa\dd b, \quad \xx  \in \RR^m
\end{align}

\paragraph{Reconstruction Formula.}
Let $\rho, \sigma \in L^2(\TT)$ satisfy the admissibility conditions
\begin{align}
    &T^{m+1} \sum_{n \neq 0} \frac{\overline{\widehat{\rho}(n)}\widehat{\sigma}(n)}{|n|^m} = 1,\label{admissible 1} \\
    & \overline{\widehat{\rho}(0)}\widehat{\sigma}(0) = 0 \quad \Longleftrightarrow \quad \int_\TT \int_\TT \overline{\widetilde{\rho}(s-t)}\sigma(t) \dd t \dd s = \left(\int_\TT \overline{\rho(s)} \dd s \right) \left( \int_\TT \sigma(t) \dd t \dd s \right) = 0. \label{admissible 2}
\end{align}
Then, for any $f \in L^1(\RR^m)$ such that $f^\sharp \in L^1(\RR^m)$, we have
\begin{align}
    \lim_{A \to \infty} \Sa[ R[f] ]
    &= f, \quad \mbox{in a.e. and } L^1.
\end{align}
Furthermore, if $\rho$ is admissible with itself, namely $T^{m+1}\sum_{n \not 0}|\widehat{\rho}(n)|^2|n|^{-m}=1$ and $\widehat{\rho}(0)=0$, then for any $f \in L^2(\RR^m)$
\begin{align}
    R^*[R[f]] = f, \quad \mbox{in } L^2.
\end{align}

\paragraph{Plancherel formula.}
Suppose $\rho \in L^\infty(\TT)$ to be admissible with itself. Then, for any $f,g \in L^2(\RR^m)$,
\begin{align}
    \iprod{R[f],R[g]}_{L^2(\RR^m\times\RR)} = \iprod{f,g}_{L^2(\RR^m)}.
\end{align}

See Appendix \ref{sec: inversion proof} for the proofs of reconstruction formula and Plancherel formula.
We remark that if $\sigma$ and $\rho$ are \emph{not} admissible with condition $T^{m+1} \sum_{n \neq 0} \frac{\overline{\widehat{\rho}(n)}\widehat{\sigma}(n)}{|n|^m} = 0$, then the reconstruction formula degenerates as
\begin{align}
    S[R[f]] = 0,
\end{align}
for any $f$. This is immediate from the proof of reconstruction formula.
This indicates that $\gamma_0 := R[f]$ becomes a null element of $S$ and thus $S$ has a non-trivial null space.

\subsection{Examples of Admissible and Non-admissible Functions}
The admissibility condition (AC) is not a strong requirement because it requires that $\sigma$ and $\rho$ are \emph{not} orthogonal to each other in the $|n|^{-m}$-weighted $\ell^2$-space. 

Let us consider the case when $\sigma$ is a periodic ReLU with period $T=1$:
\begin{align}
\sigma(t) = \begin{cases}
0 -1/8, & t \in [-1/2,0] \\
t - 1/8, & t \in [0,1/2].
\end{cases}
\end{align}
Then, the Fourier coefficients are given by 
\begin{align}
    \widehat{\sigma}(n) = \begin{cases}
0, & n \mbox{ even} \\
\frac{-2 + i \pi}{4 n^2 \pi^2}, & n \mbox{ odd}.
\end{cases}
\end{align}

\paragraph{ReLU.}
Therefore, the $\sigma$ can satisfy the admissibility condition (AC) with itself, namely $\rho_{\rm relu} = \sigma$, if it is appropriately normalized.

\paragraph{Cos.}
Recall that $\widehat{\cos m \pi t}(n)$ is always zero if $n$ is odd. Hence, $\rho_{\cos,m}(t):=\cos m \pi t$ \emph{cannot} satisfy the AC with ReLU $\sigma$ because $\overline{\widehat{\rho_{\cos,m}}}(n) \widehat{\sigma}(n) \equiv 0$ for all $n \in \ZZ$. As a result, the reconstruction  fails as $S[R[f;\rho_{\cos,m}]] = 0$ for any $m \in \NN$. 

\paragraph{Sin.}
On the other hand, $\widehat{\sin m \pi t}(n)$ is not zero for some odd $n$. Hence, $\rho_{\sin,m}(t)=\sin m \pi t$ \emph{can} satisfy the AC with ReLU $\sigma$ if it is appropriately normalized by a constant $C_m$. 

\paragraph{Difference of Admissible Functions}
Finally, let us consider the difference $\rho_{diff,1,2} := \rho_1 - \rho_2$ of two admissible functions $\rho_1$ and $\rho_2$. By the linearity of the AC, this difference \emph{cannot} satisfy the AC because $\sum_n \overline{\widehat{\rho}}(n) \sigma(n)=\sum_n \overline{\widehat{\rho_1}}(n) \sigma(n) - \sum_n \overline{\widehat{\rho_2}}(n) \sigma(n) = 0$.

Figure \ref{fig:reconst} summarizes these examples. We can visually confirm that all the admissible (a'ble) examples show different ridgelet spectrum $\gamma = R[f;\rho]$ but reproduces the original signal; and that all the non-admissible (not a'ble) examples show non-zero ridgelet spectrum but results in the null function $S[R[f]] \equiv 0$. We remark that reconstruction results are not exactly the original nor null function due to the numerical error.

\begin{figure}[H]
        \centering
        \begin{subfigure}[c]{0.2\textwidth}
                \includegraphics[width=\linewidth]{exp/sin02pt_n1000_prelu_hp050_2d.png}
        \end{subfigure}%
        \begin{subfigure}[c]{0.2\textwidth}
                \includegraphics[width=\linewidth]{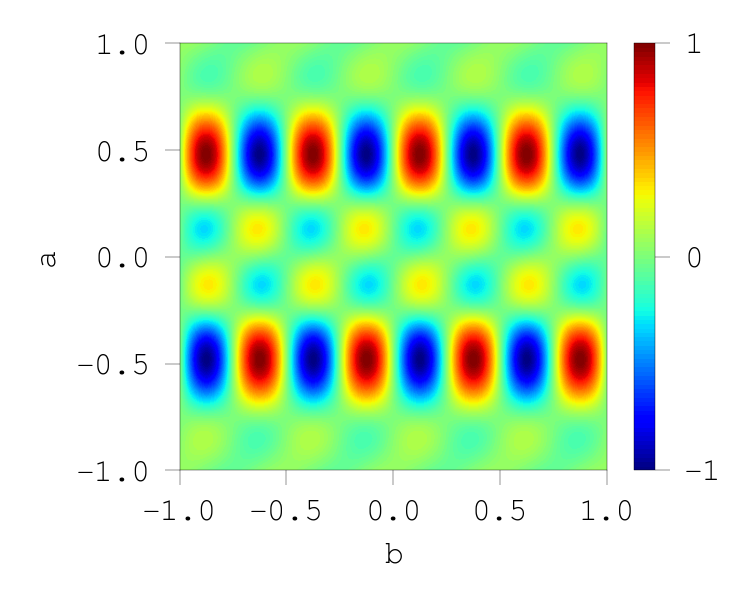}
        \end{subfigure}%
        \begin{subfigure}[c]{0.2\textwidth}
                \includegraphics[width=\linewidth]{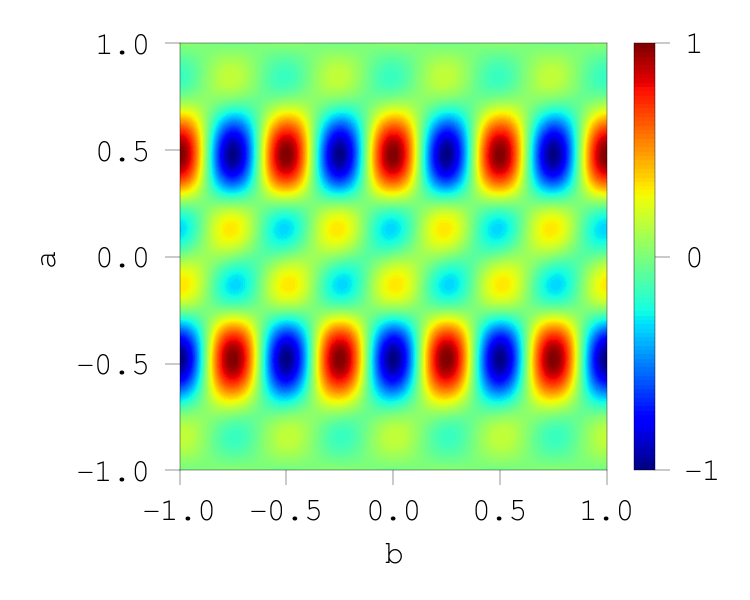}
        \end{subfigure}%
        \begin{subfigure}[c]{0.2\textwidth}
                \includegraphics[width=\linewidth]{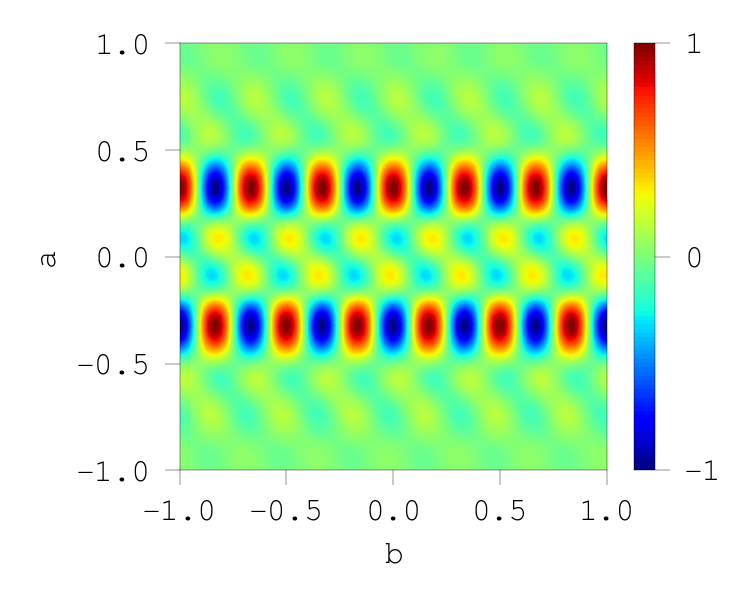}
        \end{subfigure}%
        \begin{subfigure}[c]{0.2\textwidth}
                \includegraphics[width=\linewidth]{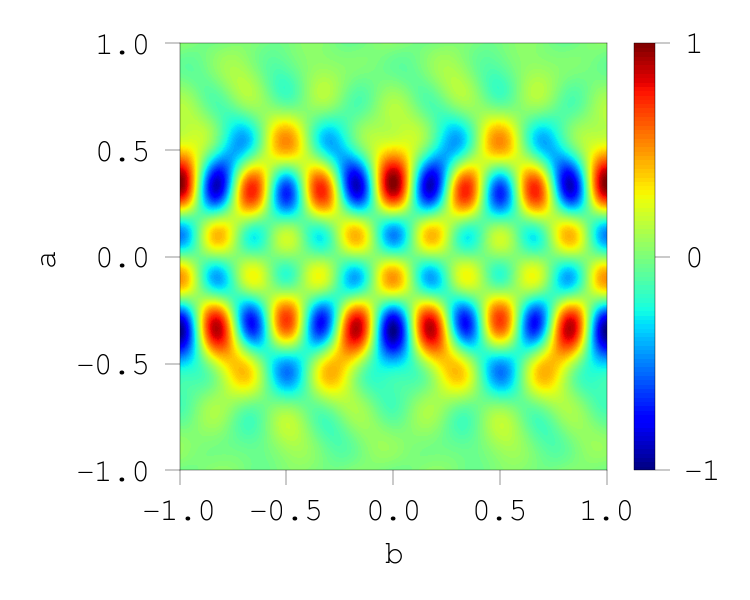}
        \end{subfigure}\\
        \begin{subfigure}[c]{0.2\textwidth}
                \includegraphics[width=\linewidth]{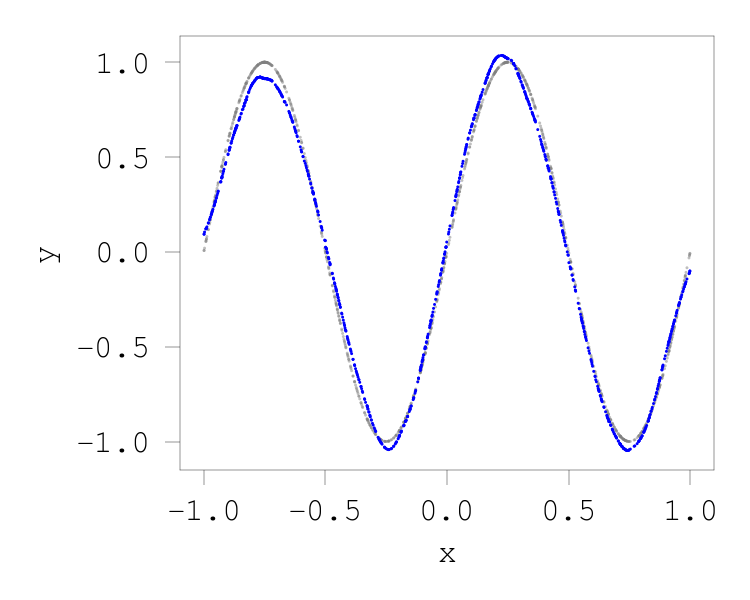}
                \caption{$\rho = \sigma$\\(a'ble)}
        \end{subfigure}%
        \begin{subfigure}[c]{0.2\textwidth}
                \includegraphics[width=\linewidth]{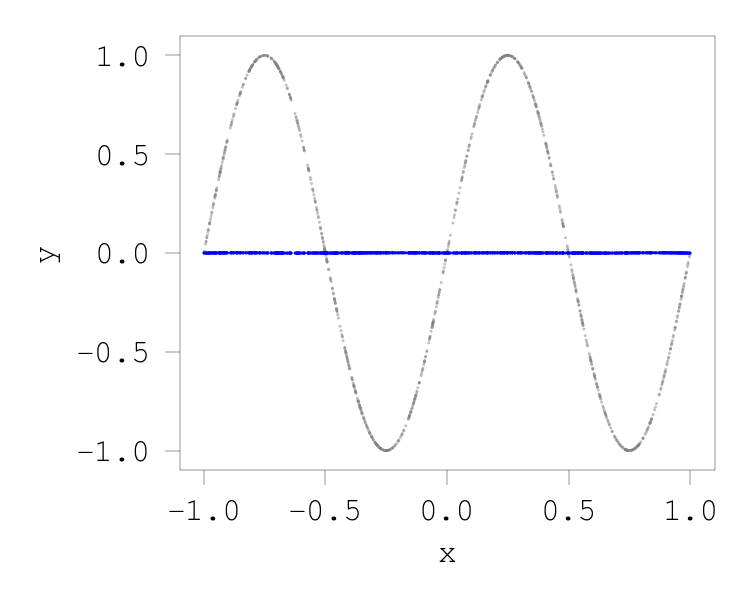}
                \caption{$\rho(t) = \cos 2 \pi t$ \\ (not a'ble)} \label{fig:reconst-cos}
        \end{subfigure}%
        \begin{subfigure}[c]{0.2\textwidth}
                \includegraphics[width=\linewidth]{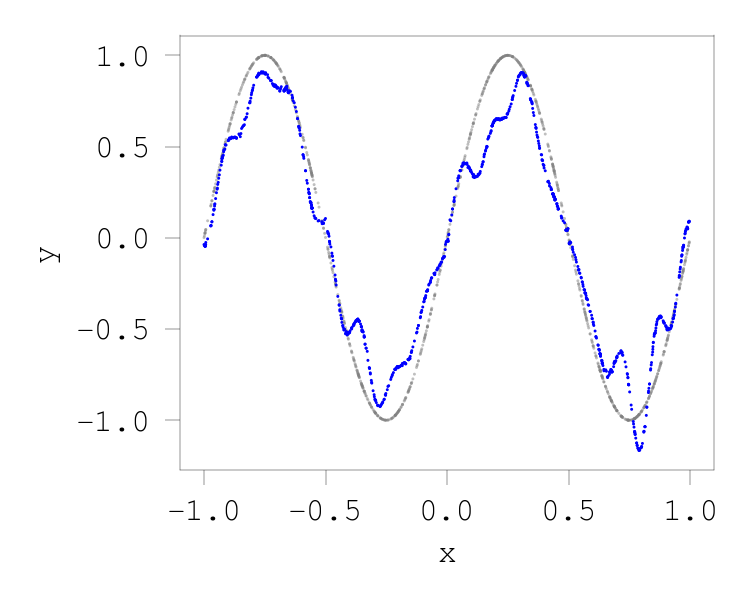}
                \caption{$\rho(t) = \sin 2 \pi t$ \\ (a'ble)}
        \end{subfigure}%
        \begin{subfigure}[c]{0.2\textwidth}
                \includegraphics[width=\linewidth]{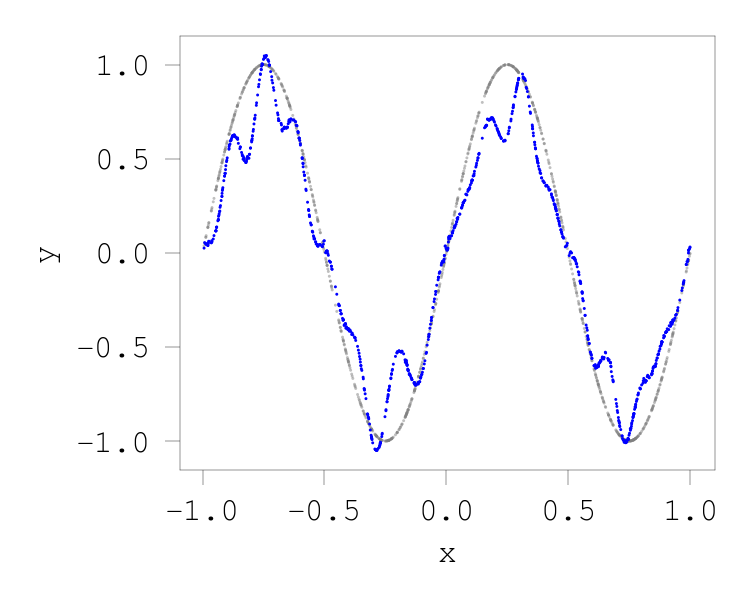}
                \caption{$\rho(t) = \sin 3 \pi t$ \\ (a'ble)}
        \end{subfigure}%
        \begin{subfigure}[c]{0.2\textwidth}
                \includegraphics[width=\linewidth]{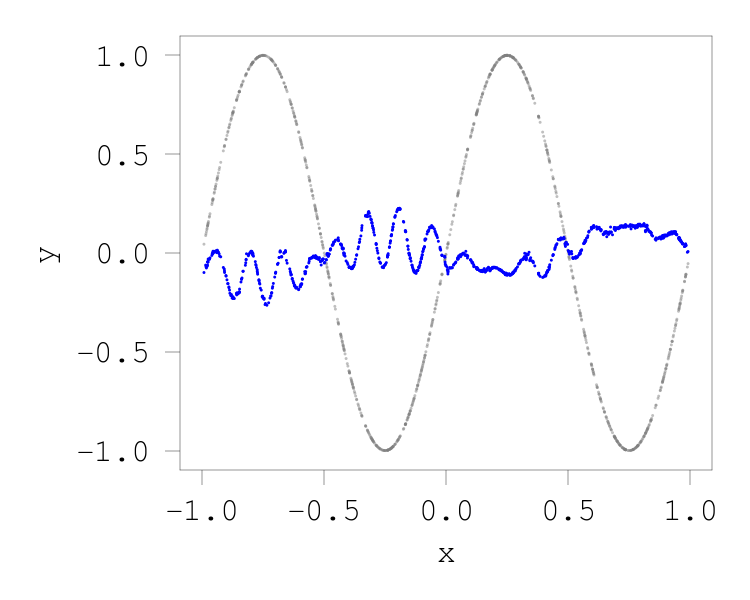}
                \caption{$\rho(t) = \sin 2 \pi t - \sin 3 \pi t$ \\ (not a'ble)}
        \end{subfigure}
        \caption{Ridgelet spectrum $R[f](a,b)$ (top) and the reconstruction result $S[R[f]](x)$ (bottom) with $\sigma = \mathrm{ReLU}$ and $f(x)=\sin 2 \pi x$ for a variety of $\rho$'s.} \label{fig:reconst}
\end{figure}

\subsection{Non-injectivity and Null Space of $S$} \label{sec:nullS}

As suggested from the previous constructive examples, there are infinitely many different solutions to the equation $S[\gamma]=f$.  Namely, suppose that two different functions $\rho_1$ and $\rho_2$ satisfy the admissibility condition, and let $\gamma_1 := R[f;\rho_1]$ and $\gamma_2 := R[f;\rho_2]$. Then, $\gamma_1 \neq \gamma_2$ but $S[\gamma_1] = f$ and $S[\gamma_2] = f$ by the reconstruction formula. This clearly implies the non-triviality of the null space $\ker S$. In general, a complete specification of $\ker S$ is very difficult.
We remark that non-injectivity occurs not only for infinite/continuous setup but also for finite/discrete setup. Indeed, Figure \ref{fig:reconst-cos} shows a clear null function $S[R[f]] \equiv 0$ as the reconstruction result, while the spectrum $R[f]$ is not a null function. These figures are obtained by numerical integration of $R[f]$ and $S[R[f]]$, which are finite dimensional approximations.

One major conclusion of this study is that if the solutions are restricted by $L^2$-regularization, then we have a unique ridgelet function $\rho = \sigma$. In general, the $L^2$-regularization provides the minimum norm solution. Therefore, we can understand that among infinitely many different solutions $\gamma = R[f;\rho]$, the $R[f;\sigma]$ achieves the minimum norm solution.

\subsection{Ridgelet Calculus} \label{sec:Rcalc}
We list some handy Fourier-like formulas for neural networks. Since regularized optimization converges to the ridgelet spectrum, if $f$ is modified, then $R[f]$ changes in accord with the following formula.
We do not use them in the main contents. 

\paragraph{Fourier Slice Theorem.}
In particular, $R$ has a Fourier expression:
\begin{align}
    R[f](\aa,b)
    &= \sum_{n=-\infty}^\infty f^\sharp(\omega_n \aa) \overline{\widehat{\rho}(n)} e^{i \omega_n b}
\end{align}
The Fourier slice theorem is originally for Radon transform (eg., see \citealp{Helgason.new}). We refer to \citet{Kostadinova2014} and \citet{sonoda2015} for other versions.
\begin{proof} 
Since
    $\frac{1}{T} \int_{-T/2}^{T/2} \overline{\rho(\aa\cdot \xx  - b)} e^{-i \omega_n b} \dd b 
    = \overline{\widehat{\rho}(n)} e^{- i \omega_n \aa\cdot \xx },$
we have
\begin{align*}
    R[f](\aa,b)
    = \sum_{n=-\infty}^\infty \left[ \int_{\RR^m} f(\xx ) \overline{\widehat{\rho}(n)} e^{-i \omega_n \aa\cdot \xx } \dd \xx  \right] e^{i \omega_n b}
    = \sum_{n=-\infty}^\infty f^\sharp(\omega_n \aa) \overline{\widehat{\rho}(n)} e^{i \omega_n b}
\end{align*}
\end{proof}

\paragraph{Ridgelet Calculus in $f$.}
\begin{align}
    &R[f(\cdot - \yy) ; \rho](\aa,b)
    = R[f ; \rho ](\aa, b + \aa \cdot \yy) \\
    &R[f(s \cdot) ; \rho](\aa,b)
    = R[f ; \rho ](\aa/s, b)/|s|^m \\
    &R[ \partial_i f ; \rho ](\aa,b)
    = a_i \partial_b R[f ; \rho](\aa,b)
    = -a_i R[f ; \rho'](\aa,b) \\
    &R[ \triangle f ; \rho](\aa,b)
    = |\aa|^2 \partial_b^2 R[f;\rho](\aa,b)
    = |\aa|^2 R[f;\rho^{(2)}](\aa,b).
\end{align}

\paragraph{Ridgelet Calculus in $\rho$.}
\begin{align}
    &R[f;\rho( \cdot - t)](\aa,b)
    = R[f ; \rho ](\aa,b+t) \\
    &R[f;\rho( s \cdot)](\aa,b) 
    = R[f ; \rho ](s\aa,sb) \\
    &R[f ; \rho' ](\aa,b)
    = -\partial_b R[f ; \rho](\aa,b)
    = -ai^{-1}R[ \partial_i f ; \rho ](\aa,b).
\end{align}

\paragraph{Convolution Theorem.}
We have
\begin{align}
    R[f * g ; \rho * \sigma](\aa,b) = \int_{\RR} R[f;\rho](\aa,b') R[g;\sigma](\aa,b-b') \dd b'.
\end{align}
\begin{proof}
According to the Fourier slice theorem,
\begin{align}
    \frac{1}{T} \int_{\TT} R[f;\rho](\aa,b) e^{-i \omega_n b} \dd b &= f^\sharp(\omega_n \aa) \overline{\widehat{\rho}(n)}.
\end{align}
Therefore,
\begin{align}
    &\frac{1}{T}\int_{\TT} R[ f * g ; \rho * \sigma](\aa,b) e^{-i\omega_n b}\dd b \\
    &= T f^\sharp(\omega_n \aa) g^\sharp(\omega_n \aa) \overline{\widehat{\rho}(n)\widehat{\sigma}(n)} \\
    &= T\left(\frac{1}{T}\int_{\TT} R[f;\rho](\aa,b) e^{-i \omega_n b} \dd b \right) \left( \frac{1}{T}\int_{\TT} R[g;\sigma](\aa,b) e^{-i \omega_n b} \dd b \right) \\
    &= \frac{1}{T}\int_{\RR} R[f;\rho](\aa,\cdot) * R[g;\sigma](\aa,\cdot) (b) e^{-i \omega_n b} \dd b. \qedhere
\end{align}
\end{proof}

\section{REGULARIZED SQUARE LOSS MINIMIZATION IN HILBERT SPACE} \label{app:tikhonov}
Let $G, F$ be Hilbert spaces endowded with the inner products $\iprod{\cdot, \cdot}_G$ and $\iprod{\cdot, \cdot}_F$, respectively, and $S : G \to F$ be a densely defined closed linear operator.

For a given $f \in F$, we find $\gamma \in G$ satisfying
\begin{align}
S[\gamma] = f.
\end{align}
For this problem, we have the following.
\begin{prop}
\label{minima for tikhonov regularization}
Let $f\in F$. Then for every $\beta >0$, we have
\begin{align}
\argmin_{\gamma \in G}\left( \| S[\gamma] - f \|_F^2 + \beta \| \gamma \|_G^2\right)
= (\beta + S^* S)^{-1} S^*[f],
\end{align}
where $S^*:F\rightarrow G$ denotes the adjoint operator of $S$.
\end{prop}
\begin{proof}
A direct computation gives
\begin{align*}
\lefteqn{\| S[\gamma] - f \|_F^2 + \beta \| \gamma \|_G^2}\\
&= \iprod{S[\gamma], S[\gamma]}_F - 2 \Re \iprod{  S[\gamma], f}_F + \iprod{f, f}_F + \beta \iprod{\gamma, \gamma}_G \\
&= \iprod{\sqrt{\beta + S^* S}[\gamma], \sqrt{\beta + S^* S}[\gamma]}_G
- 2 \Re \iprod{ \sqrt{\beta + S^* S}[\gamma], \sqrt{\beta + S^* S}^{-1} S^*[f]}_G + \iprod{f, f}_F \\
&= \| \sqrt{\beta + S^* S}[\gamma]- \sqrt{\beta + S^* S}^{-1} S^*[f]\|_G^2 
+ \iprod{f, \beta(\beta + SS^*)^{-1}[f]}
\end{align*}
Therefore, the objective functional attains the minimum at $\gstar = (\beta + S^* S)^{-1} S^*[f]$.
\end{proof}

\begin{prop} \label{limit projection}
Suppose that $\gamma_0 \in G$ satisfies $f = S[\gamma_0]$. Then,
\begin{align}
\lim_{\beta \searrow 0} (\beta + S^* S)^{-1} S^*[f] = \proj_{G \to (\ker S)^\perp} [\gamma_0].
\end{align}
\end{prop}
\begin{proof}
Using the right continuous resolution of the identity $\{ E_\mu \}_{\mu \in \RR}$ for $S^* S$,
\begin{align*}
(\beta + S^* S)^{-1} S^*[f]
&= \int_{\RR} \frac{\mu}{\beta + \mu} \dd E_\mu \gamma_0\\
&\quad \to \int_{\RR} \ind_{\RR \setminus \{ 0\} } (0) \dd E_\mu \gamma, \quad \mbox{ as } \beta \to 0 \\
&=(E_G-E_{0-})\gamma_0 \\
&=\proj_{G \to (\ker S)^\perp} \gamma_0. \qquad \qedhere
\end{align*}
Here, $=(E_G-E_{0-})\gamma_0$ follows from the projection nature of $\dd E_{\mu} \gamma_0$.
\end{proof}

\section{PROOFS}
\subsection{Propositon \ref{prop: boundedness of S}}
By the Schwartz inequality, we have
\[\| S_\lambda[\gamma] \|_{L^2(P)} \le P(\RR^m)^{1/2}\lambda(\RR\times\TT)\|\sigma\|_{L^\infty(\TT)}\cdot \|\gamma\|_{L^2(\lambda)}.\]

\subsection{Theorem \ref{thm: plancherel inversion}} \label{sec: inversion proof}

Let $\rho, \sigma \in L^2(\TT)$ satisfy the admissibility conditions (\ref{admissible 1}--\ref{admissible 2}):
\[
    T^{m+1} \sum_{n \neq 0} \frac{\overline{\widehat{\rho}(n)}\widehat{\sigma}(n)}{|n|^m} = 1, %
     \quad \overline{\widehat{\rho}(0)}\widehat{\sigma}(0) = 0. %
\]
Then, for any $f \in L^1(\RR^m)$ such that $f^\sharp \in L^1(\RR^m)$, we have
\begin{align}
    \lim_{A \to \infty} \Sa[ R[f] ]
    &= f, \quad \mbox{in a.e. and } L^1.
\end{align}
Furthermore, if $\rho$ is admissible with itself, then for any $f \in L^2(\RR^m)$
\begin{align}
    R^*[R[f]] = f, \quad \mbox{in } L^2.
\end{align}

\begin{proof}
For any $f \in L^1(\RR^m)$ such that $f^\sharp \in L^1(\RR^m)$,
\begin{align*}
    \Sa[ R[f] ](\xx )
    &= \int_{\RR^m \times \II_A^m \times \TT} f(\yy ) \overline{\rho(\aa\cdot \yy - b)} \sigma( \aa\cdot \xx  - b ) \dd \yy \dd \aa\dd b \\
    &= \int_{\RR^m \times \II_A^m} f(\yy ) \overline{\widetilde{\rho}} * \sigma (\aa\cdot (\xx - \yy)) \dd \yy \dd \aa\\
    &= \int_{\RR^m \times \II_A^m} f(\yy ) \left[ T \sum_{n \neq 0} \overline{\widehat{\rho}(n)}\widehat{\sigma}(n) \exp\left\{ \frac{2 \pi i n \aa\cdot (\xx - \yy)}{T} \right\} \right] \dd \yy \dd \aa\\
    &= \frac{1}{(2 \pi)^m} \sum_{n \neq 0} \int_{\II_{A/|n|}^m} f^\sharp(\aa) \left[ T^{m+1}  \frac{\overline{\widehat{\rho}(n)}\widehat{\sigma}(n)}{|n|^m} \right] e^{i \aa\cdot \xx} \dd \aa\\
    &\to \frac{1}{(2 \pi)^m} \int_{\RR^m} f^\sharp(\aa) \underbrace{ \left[ T^{m+1} \sum_{n \neq 0} \frac{\overline{\widehat{\rho}(n)}\widehat{\sigma}(n)}{|n|^m} \right]}_{=1} e^{i \aa\cdot \xx} \dd \aa, \quad A \to \infty \\
    &= f(\xx ). %
\end{align*}
In the forth equality, we changed the variables as $\aa' \gets n \aa$ and $\dd \aa' = |n|^m \dd \aa$. 
Since $S[R[f]](\xx)$ is absolutely convergent at any $\xx \in \RR^m$, i.e., \[ |S[R[f]](\xx)| \le \int_{\RR^m\times\II_A^m} |f(\xx-\yy)| |\widetilde{\rho}*\sigma(\aa\cdot\yy)| \dd \yy \dd \aa \le \|f\|_{L^1(\RR^m)} \| \sigma \|_{L^\infty(\TT)} \| \sigma \|_{L^\infty(\TT)} |\II_A^m|,\] we can use Fubini freely before the limit.
The final equality follows from the Fourier inversion theorem for $L^1$-functions. When $\rho$ is admissible with itself, then we can extend $R^*[R[f]] =f$ for any $f \in \Fx$ by the bounded extension using the Plancherel formula: $\| R[f] \|_{L^2(\RR^m\times\RR)}^2 = \iprod{ R^*[R[f]],f }_{L^2(\RR^m)}=\| f \|_{L^2(\RR^m)}$ for $f \in L^1(\RR^m)\cap L^2(\RR^m)$.

The limit $A \to \infty$ is justified by the dominated convergence theorem as follows.
For each $n \in \ZZ$,
write $g(n) := T^{m+1} \overline{\widehat{\rho(n)}}\widehat{\sigma}(n)|n|^{-m}$ if $n \neq 0$, and $g(0)=0$. By the assumption, $\sum_{n \in \ZZ}g(n)=1$. 
Let 
\begin{align}
h_A^N(\xx) &:= \frac{1}{(2 \pi)^m} \sum_{0< |n| \le N} \int_{\II_{A/n}^m} f^\sharp(\aa) g(n) e^{i\aa\cdot\xx}\dd\aa, \quad \mbox{and} \\
h(\xx) &:= \frac{1}{(2 \pi)^m} \int_{\RR^m} f^\sharp(\aa) \left[ \sum_{n \in \ZZ} g(n) \right] e^{i\aa\cdot\xx}\dd \aa = f(\xx).
\end{align}
Here $h=f$ because of the admissibility and the Fourier inversion. Then, (i) $h_A^N(\xx)$ is (uniformly) dominated: $(2\pi)^m |h_A^N(\xx)| \le \sum_{n\in\ZZ} \int_{\RR^m} |f^\sharp(\aa)||g(n)|\dd \aa = \| f^\sharp \|_{L^1(\RR^m)}$; and (ii) $h_A^N(\xx) \to h(\xx)$ for a.e. $\xx \in \RR^m$ because 
\begin{align}
    (2\pi)^m|h_A^N(\xx) - h(\xx)| 
    &= \Bigg|\left( \sum_{N \le |n|} \int_{\RR^m} + \sum_{0<|n|\le N} \int_{\RR^m \setminus \II_{A/n}^m} \right) f^\sharp(\aa) g(n) e^{i\aa\cdot\xx}\dd \aa \Bigg| \\
    &\le (2\pi)^m |f(\xx)| \Bigg|\sum_{N\le|n|}\ g(n) \Bigg| + \Bigg|\sum_{n =1}^N g(n) \Bigg| \| f^\sharp \|_{L^1(\RR^m \setminus \II_A^m)},
\end{align}
and both terms tends to $0$ as $A,N \to \infty$. Therefore, we have $\| h_A^N - h \|_{L^1(\RR^m)} \to 0$ as $A,N \to \infty$.
\end{proof}

\subsection{Corollary \ref{cor: universality}}
For $n>0$, let $\tau_n: \TT \rightarrow \RR$ be a periodic function such that $\tau^{(n)}=\sigma$, that is defined through the Fourier coefficients $\widehat{\tau_n}(0) =0 $ and $\widehat{\tau_n}(m)=\widehat{\sigma}(m)/(2\pi i m)^m$ for $m\in \ZZ \setminus \{0\}$.
Here we use $\widehat{\sigma}(0)=0$, one of the admissible conditions in Assumption \ref{asm: admissible condition}.
For $\nn=(n_1,\dots,n_m) \in \NN^m$ and $\aa = (a_1,\dots, a_d)$, we defined $|\nn|=\sum_{i=1}^m n_i$ and $|\aa|^{\nn}:=\prod_{i=1}^m |a_i|^{n_i}$.
By the integration by parts, we have
\begin{align*}
    \frac{1}{|\aa|^{\nn}}\int \partial^{\nn}f(\xx)\tau_{|\nn|}(\aa \cdot \xx - b)\dd{P}(\xx) &= \int f(\xx)\sigma(\aa \cdot \xx -b) \dd{P}(\xx) = R[f](\aa,b).
\end{align*}
Thus, $R[f] \in L^1(\RR^m\times \TT)\cap L^\infty(\RR^m\times\TT)$.
Therefore, by the strong law of large numbers of Banach space valued random variables (Proposition \ref{prop:SLLN}), we have the desired result.

\subsection{Lemma \ref{lem: limit of minimizers}}
For simplisity, we denote by $\gamma^*[P]$ and $\gamma^*[P_N]$ the minimizers $\gamma^*[f;P,\lambda,\beta]$ and $\gamma^*[f;P_N,\lambda,\beta]$, respectively.
We denote by $R_P: L^2(P) \rightarrow L^2(\lambda)$ (resp $R_{P_N}: L^2(P_N) \rightarrow L^2(\lambda)$) the adjoint operators of $S_\lambda: L^2(\lambda)\rightarrow L^2(P)$ (resp. $S_\lambda: L^2(\lambda)\rightarrow L^2(P_N$)). More precisely,
\begin{align*}
    R_P[h](\aa,b) &= \int h(x)\sigma(\aa\cdot\xx - b ) \dd P(\xx),\\
    R_{P_N}[h](\aa,b) &= \frac{1}{N} \sum_{i=1}^N h(\xx_i) \sigma(\aa \cdot \xx_i - b ).
\end{align*}
Proposition \ref{minima for tikhonov regularization}, we have
\begin{align*}
    \gamma^*[P] &= (\beta + R_PS_\lambda)^{-1} R_P[f], \\
    \gamma^*[P_N] &= (\beta + R_{P_N}S_\lambda)^{-1} R_{P_N} [f].
\end{align*}
We denote by $T$ and $T_N$ the operators $R_PS_\lambda$ and $R_{P_N}S_\lambda$ on $L^2(\lambda)$, respectively.
We note that $T$ and $T_N$ are explitly described as follows:
\begin{align*}
    T[\gamma](\aa,b) &= \int \gamma(\aa', b') \int \sigma(\aa\cdot\xx -b)\sigma(\aa' \cdot \xx - b') \dd P(\xx) \dd\lambda(\aa',b')\\
    T_N[\gamma](\aa,b) &= \int \gamma(\aa', b') \frac{1}{N} \sum_{i=1}^N \sigma(\aa\cdot\xx_i -b)\sigma(\aa' \cdot \xx_i - b') \dd\lambda(\aa', b')
\end{align*}
Then we have
\begin{align*}
    &\|\gamma^*[P] - \gamma^*[P_N] \|_{L^2(\lambda)} \\
    &\le \|(\beta + T_N)^{-1}(R_{P_N} - R_P)[f]\| + \|\left((\beta+T)^{-1} - (\beta + T_N)^{-1}\right)R_P[f]\|\\
    & \le \|(\beta+T_N)^{-1}(R_{P_N}-R_P)[f]\| + \|(\beta + T_N)^{-1}(R_{P_N} - R_P) S_\lambda(\beta + T)^{-1}R_P[f]\|
\end{align*}
Denote $h:= S_\lambda (\beta + T)^{-1}R_P[f]$. Then we have 
\begin{align*}
    \|\gamma^*[P] - \gamma^*[P_N] \|_{L^2(\lambda)} \le \beta^{-1}\|R_{P_N}[f] - R_P[f]\| + \beta^{-1}\|R_{P_N}[h]- R_P[h]\|.
\end{align*}
Since $R_{P_N}[f]$ and $R_{P_N}[h]$ are an $L^2(\lambda)$-valued random variables; their expectation are $R_P[f]$ and $R_P[h]$;
and $\EE[\| R_{P_N}[f] - R_P[f] \|_{L^2(\lambda)}] \le \EE[\| R_{P_N}[f] \|_{L^2(\lambda)}] + \| R_P[f] \|_{L^2(\lambda)} \le C_0^{1/2} \| \sigma \|_\infty \| f \|_{L^2(P)} + \| R_P[f] \|_{L^2(\lambda)} < \infty$ with $C_0 := \lambda(\II_A^m\times\TT)$, and similarly, $\EE[\| R_{P_N}[h] - R_P[h] \|_{L^2(\lambda)}] C_0^{1/2} \| \sigma \|_\infty \| h \|_{L^2(P)} + \| R_P[h] \|_{L^2(\lambda)} < \infty$.
Thus by the strong law of large numbers for Banach spaces (Proposition~\ref{prop:SLLN}), $R_N[f] \to R_P[f]$ and $R_N[h] \to R_P[h]$ a.s. respectively, and we have a desired conclusion.

\begin{prop}[{\citealp[Corollary~7.10]{Ledoux1991}}] \label{prop:SLLN}
Let $X$ be a Borel random variable with values in a separable Banach space $B$; let $\{ X_i \}_{i \in [n]}$ be i.i.d. copies of $X$; and let $S_n := \sum_{i=1}^n X_i$. Then, $S_n/n \to 0$ almost surely if and only if $\EE\|X\|<\infty$ and $\EE X=0$.
\end{prop}

\subsection{Theorem \ref{thm: minimizer = ridgelet}}
Here, we define 
\[K((\aa,b),(\aa',b')):=\int \sigma(\aa\cdot\xx-b)\sigma(\aa'\cdot\xx-b')\dd{P(\xx)}.\]
and for $(\aa,b)\in\RR^m\times\TT$, we define 
\[\mcT_A[\gamma](\aa,b):=\int \gamma(\aa',b')K((\aa',b'),(\aa,b))\dd\mu_A(\aa', b').\]
We define a bounded absolutely integrable function $w_A$ by
\[w_A(\xx):=\sum_{n\neq 0}\frac{|\widehat{\sigma}(n)|^2}{|n/T|^m}\mathbf{1}_{[-nA/2T,nA/2T]^m}(\xx).\]
\begin{lem}
The correspondence $\xx\mapsto \sigma_\xx$ is bounded and continuous mapping from $\RR^m$ to $L^2(\mu_A)$.
\end{lem}
\begin{proof}
By using a standard usage of the molifier, we may assume $\sigma$ is continuous function, thus we immediately see the continuity.   The boundedness is obvious since $\sigma$ is bounded, and $\mu_A$ is a finite measure.
\end{proof}
\begin{cor}
Let $f\in L^1(\RR^m,\dd\xx)$, and let $\mathcal{L}$ be a bounded linear operator on $\Git$. Then for any $A'>0$, $\int f(x)\mathcal{L}[\sigma_\xx]\dd\xx$ is a well-defined elemlent in $L^2(\II_{A'}^m\times\TT,\dd\aa\dd{b})$ and satisfy for any $h\in L^2(\II_{A'}^m\times\TT,\dd\aa\dd{b})$, 
\[\bigg\langle \gamma, \int f(x)\mathcal{L}[\sigma_\xx]\dd\xx\bigg\rangle_{L^2(\mu_{A'})}=\int f(\xx)\big\langle \gamma,\mathcal{L}[\sigma_{\xx}]\big\rangle_{L^2(\mu_{A'})}\dd\xx.\]
\end{cor}
\begin{lem}
\label{lem: variant of Planchrel for ridglet}
For $g\in L^1(\RR^m,\dd\xx)$, we have
\[\left\|\int g(\xx)\mcT_A[\sigma_{\xx}]\dd\xx\right\|_{\Grt)}
=\left\|[g^\flat w_A]^{\sharp}p\right\|_{\Fx}.\]
\end{lem}
\begin{proof}
Put $\phi_A(x):=w_A^{\sharp}(x)$.
Since 
\[\mcT_A[\sigma_\xx](\aa,b)=\int \sigma_{\aa,b}(\yy)p(\yy) w_A^{\sharp}(\xx-\yy)\dd{\yy},\]
for $B>0$, by direct computation, we have
\begin{align*}
    &\left\|\int g(\xx)\mcT_A[\sigma_{\xx}]\dd\xx\right\|_{L^2(\II_B^m\times\TT,\dd\aa\dd{b})}^2\\
    &=\int g(\xx)g(\yy)p(\zz)p(\ww) w_B^{\sharp}(\zz-\ww)w_A^{\sharp}(\xx-\zz)w_A^{\sharp}(\yy-\ww)\dd\xx\dd\yy\dd\ww\dd\zz\\
    &=\int [g*w_A^{\sharp}]p(\ww)[g*w_A^{\sharp}]p(\zz)w_A{\sharp}(\ww-\zz)\dd\ww\dd\zz\\
    &=\int \left|\left([g^\flat w_A]^{\sharp}p\right)^{\flat}(\xx)\right|^2 w_B(\xx)\dd\xx.\\
\end{align*}
By taking $B$ to $\infty$, we have the formula.
\end{proof}

\begin{cor}
\label{cor: limit of variant of ridgelet}
For any $g\in \Fx$, the integral $\int g(\xx)\mcT_A[\sigma_{\xx}]\dd\xx$ is well-defined in the similar manner with the Fourier transform. Moreover, we have
\[\lim_{A\rightarrow\infty}\left\|\int g(\xx)\mcT_A[\sigma_{\xx}]\dd\xx - gp\right\|_{\Grt}=0.\]
\end{cor}
\begin{thm}
Let $P$ be an absolutely continuous finite Borel measure on $\RR^m$ with density function $p$.
Let $f\in \Fp$.  Assume $p$ is bounded and $f\in \Fx$.  Then we have
\begin{align}
    \gstar[f]=R\big[\frac{pf}{\beta+p}\big]+\Delta_{\beta,A}[f],
\end{align}
where $\Delta_{\beta,A}[f]$ is an element of $\Git$ such that 
\[\lim_{A\rightarrow\infty}\left\Vert \Delta_{\beta,A}[f]\right\Vert_{\Git}=0.\]
\end{thm}

\begin{proof}
By the theory of the Tikhonov regularization, $\gstar$ is explicitly described as follows:
\[\gstar[f]=(\beta+\mcT_A)^{-1}S_{\mu_A}^*[f],\]
where we write $\mcT_A:=S_{\mu_A}^*\Sa$. We denote $\frac{pf}{\beta+p}$ by $g$.
Let
\begin{align*}
    W_A&:=\int g(p\sigma_{\xx}-\mcT_A[\sigma_{\xx}])\dd\xx\\
    &=R[gp]-\int g \mcT_A[\sigma_{\xx}]\dd\xx.
\end{align*}
Then, we have
\begin{align*}
    (\beta + \mcT_A)R[g] + W_A &= (\beta + \mcT_A)R[g] + R[gp] - \int g \mcT_A[\sigma_{\xx}]\dd\xx \\
    & = \beta R[g] + \int g \mcT_A[\sigma_{\xx}]\dd\xx + R[gp] - \int g \mcT_A[\sigma_{\xx}]\dd\xx\\
    & = R[(\beta + p)g ]\\
    & = R[pf]\\
    & = S_{\mu_A}^*[f].
\end{align*}
Thus, we have
\[(\beta+S_{\mu_A}^*\Sa)^{-1}S_{\mu_A}^*[f]=R[g]+(\beta+\mcT_A)^{-1}W_A.\]
By Lemma \ref{lem: variant of Planchrel for ridglet}, we have $W_A\in \Grt$.  By Corollary \ref{cor: limit of variant of ridgelet}, we see that $W_A\rightarrow 0$ in $\Grt$.  Therefore, we define  
\[\Delta_{\beta,A}[f]:=(\beta+\mcT_A)^{-1}W_A\]
and the limit of $\Delta_{\beta,A}[f]$ is zero as $A\rightarrow \infty$.
\end{proof}

\subsection{Lemma \ref{lem: support collapse}}
\begin{proof}

By the Schwartz inequality, we have
\begin{align*}
    \| \gamma \|_{L^1(\lambda)} &= \int |\gamma(\aa,b)|\dd\lambda(\aa,b),\\
    &\le \lambda({\rm supp}(\gamma)) \cdot \|\gamma\|_{L^2(\lambda)}.
\end{align*}
Thus we have $\|\gamma\|_{L^2(\lambda)} >C/\lambda({\rm supp}(\gamma))$.
\end{proof}

\subsection{Lemma \ref{lem: convergence with over parametrization}}
Here we prove the following statement:
\begin{lem}[Lemma \ref{lem: convergence with over parametrization}]
\label{applem: over parametrized convergence}
Let $f$ be a bounded continuous function on $\RR^m$.
For every $d \in \NN$, let $\lambda_d:=\frac{C_0}{d}\sum_{i=1}^d\delta_{(\aa_i,b_i)}$ with $(\aa_i, b_i) \in \II_A^m \times \TT$ and $C_0 := (2A)^mT$. Assume that $\lambda_d$ weakly converges to $\mu_A$.
Here, the weak convergence is in the sense that $\int_{\II_A^m \times \TT} h \dd \lambda_d \to \int h \dd\mu_A$ for any bounded continuous function $h$ on $\II_A^m \times \TT$.
Then, as $d\rightarrow\infty$, we have
\[\gamma^*[f;P_N,\gamma_d,\beta_d] \longrightarrow \gamma^*[f;P_N,\gamma,\beta].\]
\end{lem}
\begin{proof}
We denote by $z_i$ the point $(\aa_i,b_i)$.
For simplicity, we write $S_A:=S_{\mu_A}$ and $S_d=S_{\lambda_d}$.  
Let $\mcT_A:=S_A^*S_A$, and define a linear operator $\mcT_d:=S_d^*S_d$ on $\Gd$, namely, 
\begin{align}
\mcT_A[\gamma](\aa,b) &:=S_A^*[\Sa[\gamma]](\aa,b) = \int_{\II_A^m \times \TT} \gamma(\aa',b') K( (\aa,b), (\aa',b') ) \dd \aa' \dd b'\\
\mcT_d[\gamma](\aa,b) &:=S_d^*[\Sd[\gamma]](\aa,b) = \frac{C_0}{d} \sum_{i=1}^d \gamma(\aa_i,b_i) K( (\aa,b), (\aa_i,b_i) ),
\end{align}
where, we denote $K((\aa,b),(\aa',b'))=\int \sigma(\aa\cdot\xx-b)\sigma(\aa'\cdot\xx-b')\dd{P(\xx)}$.
We denote by $G_d$ (resp. $G$) the minimizer $\gamma^*[f;P_N,\lambda_d,\beta_d]$ (resp.  $\gamma^*[f;P_N, \mu_A, \beta]$).
Since for any Riemann-integrable function $\gamma\in \Gd$, $\mcT_A[\gamma]$ is bounded and continuous almost everywhere, $G$ is also bounded and continuous almost everywhere because it satisfies $G=\beta^{-1}(G-\mcT_A[G])$.  By direct computation, we have
\begin{align}
&\|G-G_d\|^2_{\Gd}\nonumber\\
&\le \beta_d^{-2}\Vert(\beta_d+\mcT_d)[G]-\Sd^*[f]\Vert^2_{\Gd}\nonumber\\
&= \beta_d^{-2}\|(\mcT_A-\mcT_d)G + (\beta_d-\beta)G + (S_A^*-\Sd^*)[f]\|^2_{\Gd} \nonumber\\
&\le 2\beta_d^{-2}\| \mcT_d[G]-\mcT_A[G]\|^2_{\Gd} + 2\beta_d^{-2}|\beta_d - \beta|^2\cdot \|G\|_{L^2(\lambda_d)}^2\nonumber\\
&=2\beta_d^{-2}|\beta_d - \beta|^2\cdot \|G\|_{L^2(\lambda_d)}^2+ \frac{2C_0}{\beta_d^2d}
\sum_i^d\left|\frac{C_0}{d}\sum_{j=1}^d G(z_j)K(z_i,z_j)-\int G(\aa,b)K(z_i,(\aa,b)) \dd\mu_A(\aa,b)\right|^2\nonumber\\
&\le 2\beta_d^{-2}|\beta_d - \beta|^2\cdot \|G\|_{L^2(\lambda_d)}^2 + \frac{2}{\beta_d^{2}}\left(U^1_d - 2 U^2_d +U^3_d\right).\nonumber
\end{align}
For the first inequality, we use $\left\|(\beta_d+\mcT_d)^{-1}\right\|\le \beta_d^{-1}$.
For the third equality,  we use the inequality $(a+b)^2 \le 2(a^2+b^2)$ and the formula $\Sd^*[f](\aa,b)=\int f(x)\sigma(\aa\cdot\xx-b)\dd{P_N(x)}=S_A^*[f](\aa,b)$ for $(\aa,b)\in\{z_i\}_{i=1}^d$.
For the last term,
we put
\begin{align*}
    U^1_d &:= \frac{C_0^3}{d^3}\sum_{i,j,k=1}^dG(z_i)G(z_j)K(z_k,z_i)K(z_k,z_j),\\
    U^2_d &:= \frac{C_0^2}{d^2}\sum_{i,j=1}^dG(z_i)K(z_j, z_i)\int G(\aa,b)K(z_j,(\aa,b))\dd\mu_A(\aa,b),\\
    U^3_d &:= \frac{C_0}{d}\sum_{i=1}^d\left(\int G(\aa,b)K(z_i,(\aa,b))\dd\mu_A(\aa,b)\right)^2.
\end{align*}
By the Riemann integrability of $G(\cdot)K(\cdot,\cdot)$, for $i=1,2,3$, as $d\rightarrow\infty$, we have
\[U_d^i \longrightarrow \int\int\int G(z)G(w)K(v,z)K(v,w)\dd\mu_A(z)\dd\mu_A(w)\dd\mu(v).\]
Thus we see that $\|G - G_d\|^2_{L^(\lambda_d)} \rightarrow 0$ as $d\rightarrow\infty$.

\end{proof}

\subsection{Theorem \ref{thm: convergence of finite NN}}

Here, we prove the following statement:
\begin{cor}[Theorem \ref{thm: convergence of finite NN}]
Let $\{ \gamma^*_{N,d} \}_{d=1}^\infty$ be a sequence of ERMers. Impose Assumption \ref{asm: support} on the hidden parameter distributions $\lambda_d$ of $\gamma^*_{N,d}$, namely, $\lambda_d$ weakly converges to $\mu_A$. Assume $\beta_d\rightarrow \beta$ as $d\rightarrow\infty$.
Then, for any bounded continuous function $h$ on $\II_A^m\times\TT$, we have
\begin{align}
\lim_{N\rightarrow\infty}\lim_{d\rightarrow\infty}\int h \gamma^*_{N,\lambda_d}\dd{\lambda_d} = \int h \gamma^* \dd\mu_A.    
\end{align}
Here the limit with respect to $N$ is in the sense of $P$-a.s. convergence.
\end{cor}
\begin{proof}
Put $\gamma^*_{N,\infty}:=\gamma^*[f;P_N,\mu_A,\beta]$.
Then, by Lemma \ref{applem: over parametrized convergence}, we have 
\[\left| \int h\gamma^*_{N,d} - h\gamma^*_{N,\infty} \dd\lambda_d \right| \le \|h\|_{L^\infty(\mu_A)}\cdot \|\gamma^*_{N,d}-\gamma^*_{N,\infty} \|_{L^2(\lambda_d)} \rightarrow 0.\]
Since $h\gamma^*_{N,\infty}$ is Riemann integrable, we see that $\lim_{d\rightarrow\infty}\int h\gamma^*_{N,\infty} \dd\lambda_d = \int h\gamma^*_{N,\infty} \dd\mu_A$.
Thus, we have 
\[\lim_{d\rightarrow\infty} \int h\gamma^*_{N,d} \dd\lambda_d = \int h\gamma^*_{N,\infty} \dd\mu_A.\]
On the other hand, by Lemma \ref{lem: limit of minimizers}, we have $\gamma^*_{N,\infty} \rightarrow \gamma^*$ $P$-a.e. as $N\rightarrow\infty$.
By the same argument employing the Schwartz inequality as above, we have
\[\lim_{N\rightarrow \infty} \int h\gamma^*_{N,\infty} \dd\mu_A = \int h\gamma^* \dd\mu_A~~\text{$P$-a.e.}.\] 
\end{proof}

\newpage
\section{FURTHER RESULTS ON NUMERICAL SIMULATION} \label{app:experiments}

\paragraph{Experiment~1.}
In order to see the differences among activation functions, we conduct the experiment on a common dataset: $y_i = \sin 2 \pi x_i$  with three activation functions: Gaussian $\sigma(t) = \exp( -|k t|^2 )$ with scale $k=6$, Tanh $\sigma(t) = \tanh(kt)$ with scale $k=6$, and ReLU $\sigma(t) = \max\{ 0, t\}$. In order to cover a characteristic part in the period $\TT = [-1/2,1/2)$, we introduced the scale parameter $k$ for Gaussian and Tanh. %
As a result, all the three $\sigma$s have period $T=1$. If an activation function is periodic with period $T$, then the spectrum is periodic in $b$ with period $T$ because
\begin{align}
    R[f;\sigma](a,b) = R[f;\sigma( \cdot - T )](a,b) = R[f;\sigma](a,b+T).
\end{align}
We can verify that our theory accepts a variety of activation functions.
For all the three settings, we trained $s=1000$ networks,
each single network has $d=100$ hidden units, and the weight decay rate and learning rate were set to $\beta = 0.001$ and $\eta = 0.01$ respectively.

\begin{figure}[H]
        \centering
        \begin{subfigure}[c]{0.3\textwidth}
                \includegraphics[width=\linewidth]{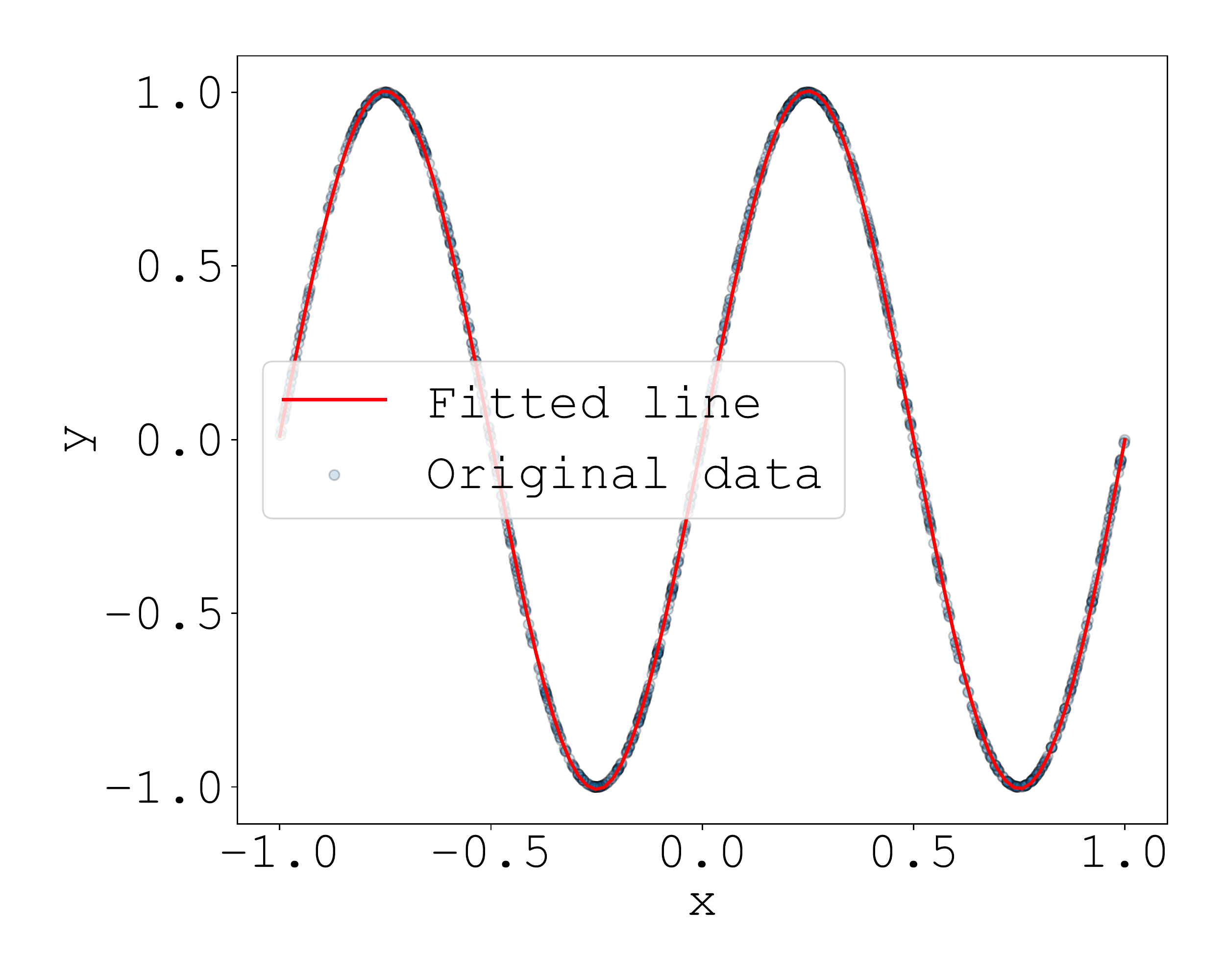}
        \end{subfigure}%
        \begin{subfigure}[c]{0.35\textwidth}
                \includegraphics[width=\linewidth]{exp/sin02pt_n1000_pgauss_s60_hp050_h100_sgd_wd001_nn0500.png}
        \end{subfigure}%
        \begin{subfigure}[c]{0.35\textwidth}
                \includegraphics[width=\linewidth]{exp/sin02pt_n1000_pgauss_s60_hp050_2d.png}
        \end{subfigure}\\
        \begin{subfigure}[c]{0.3\textwidth}
                \includegraphics[width=\linewidth]{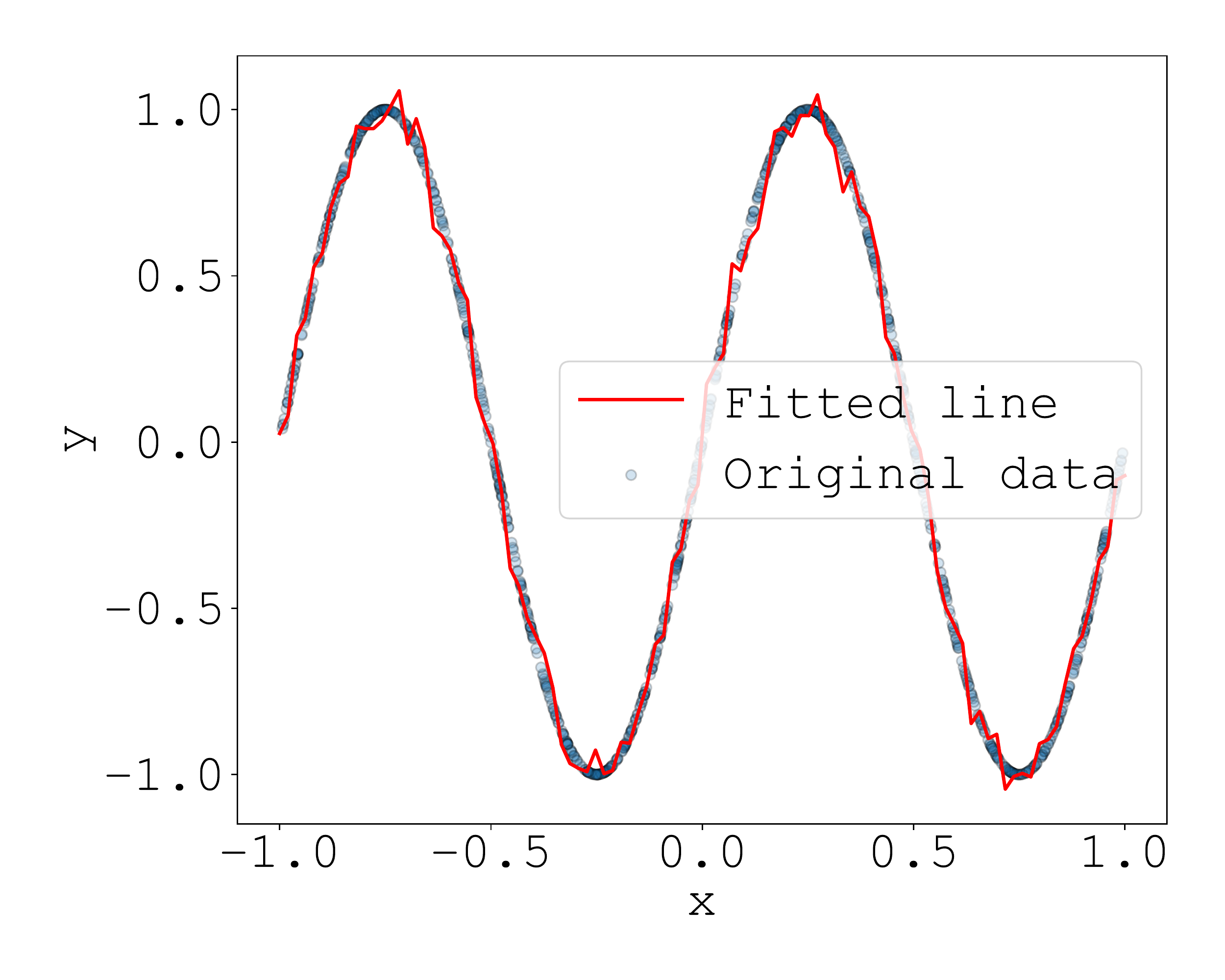}
        \end{subfigure}%
        \begin{subfigure}[c]{0.35\textwidth}
                \includegraphics[width=\linewidth]{exp/sin02pt_n1000_ptanh_s60_hp050_h100_sgd_wd001_nn0500.png}
        \end{subfigure}%
        \begin{subfigure}[c]{0.35\textwidth}
                \includegraphics[width=\linewidth]{exp/sin02pt_n1000_ptanh_s60_hp050_2d.png}
        \end{subfigure}\\
        \begin{subfigure}[c]{0.3\textwidth}
                \includegraphics[width=\linewidth]{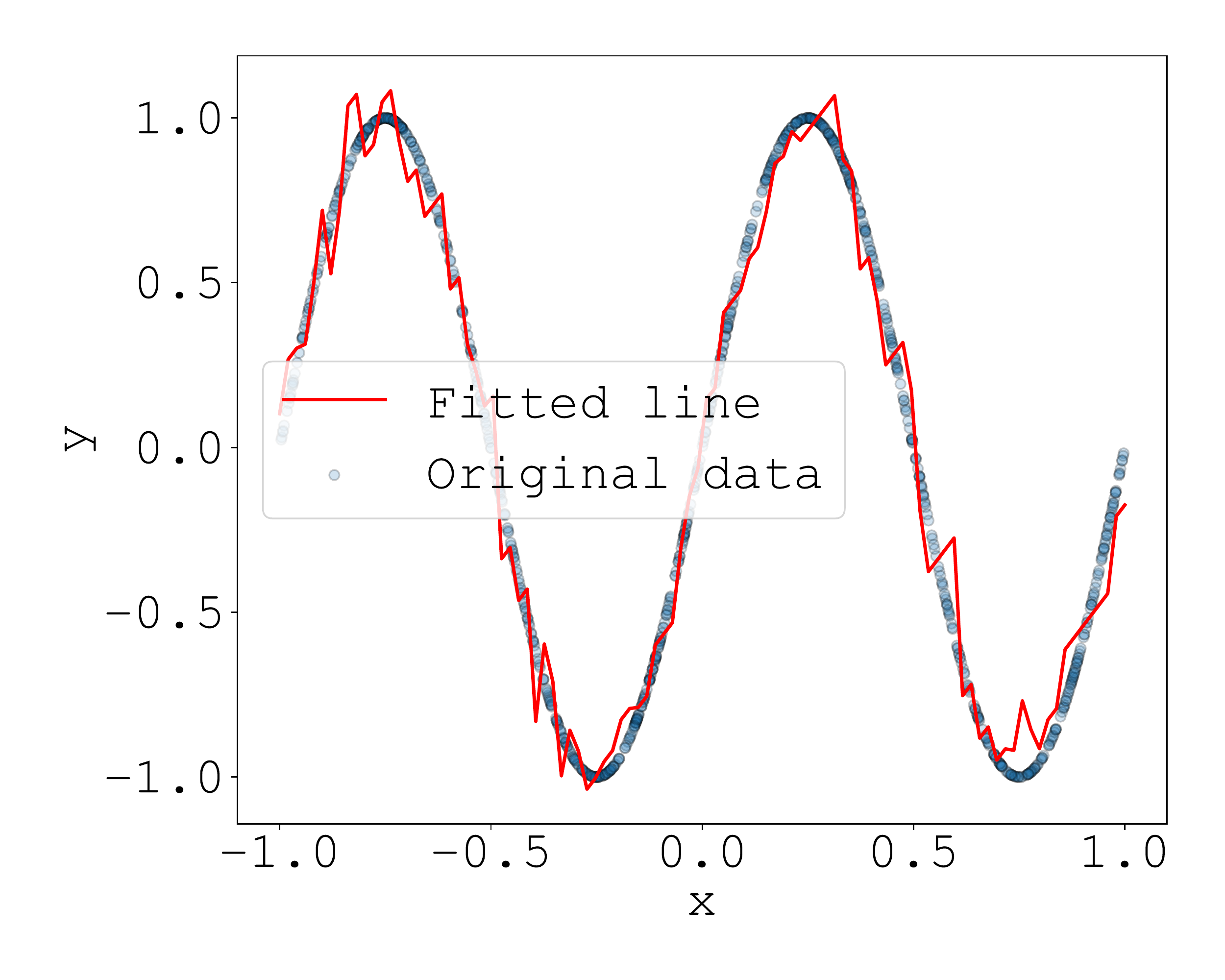}
        \end{subfigure}%
        \begin{subfigure}[c]{0.35\textwidth}
                \includegraphics[width=\linewidth]{exp/sin02pt_n1000_prelu_hp050_h100_sgd_wd001_nn1000.png}
        \end{subfigure}%
        \begin{subfigure}[c]{0.35\textwidth}
                \includegraphics[width=\linewidth]{exp/sin02pt_n1000_prelu_hp050_2d.png}
        \end{subfigure}
        \caption{$f(x) = \sin 2 \pi x$, $\sigma(z) = \exp (-(kz)^2/2), \tanh (kz), \relu (z)$ (from top to bottom)}\label{fig:sin2p}
\end{figure}

\newpage
\paragraph{Experiment~2.}
In order to focus on a structure as a ridgelet spectrum, we prepared translated datasets $y_i = \exp (-|x_i-\mu|^2/2)$ with $\mu = -0.5, 0, 0.5$. We employ the periodic ReLU on $\TT = [-1/2,1/2)$ for the activation function. According to the ridgelet transform, it satisfies the \emph{translation (time-shifting) property}:
\begin{align}
    R[f( \cdot - y )](a,b) = R[f](a,b + a \cdot y).
\end{align}
We can clearly observe this relation in the scatter plots.
For all the three settings, we trained $s=1000$ networks,
each single network has $d=100$ hidden units, and the weight decay rate and learning rate were set to $\beta = 0.001$ and $\eta = 0.01$ respectively.

\begin{figure}[H]
        \centering
        \begin{subfigure}[c]{0.3\textwidth}
                \includegraphics[width=\linewidth]{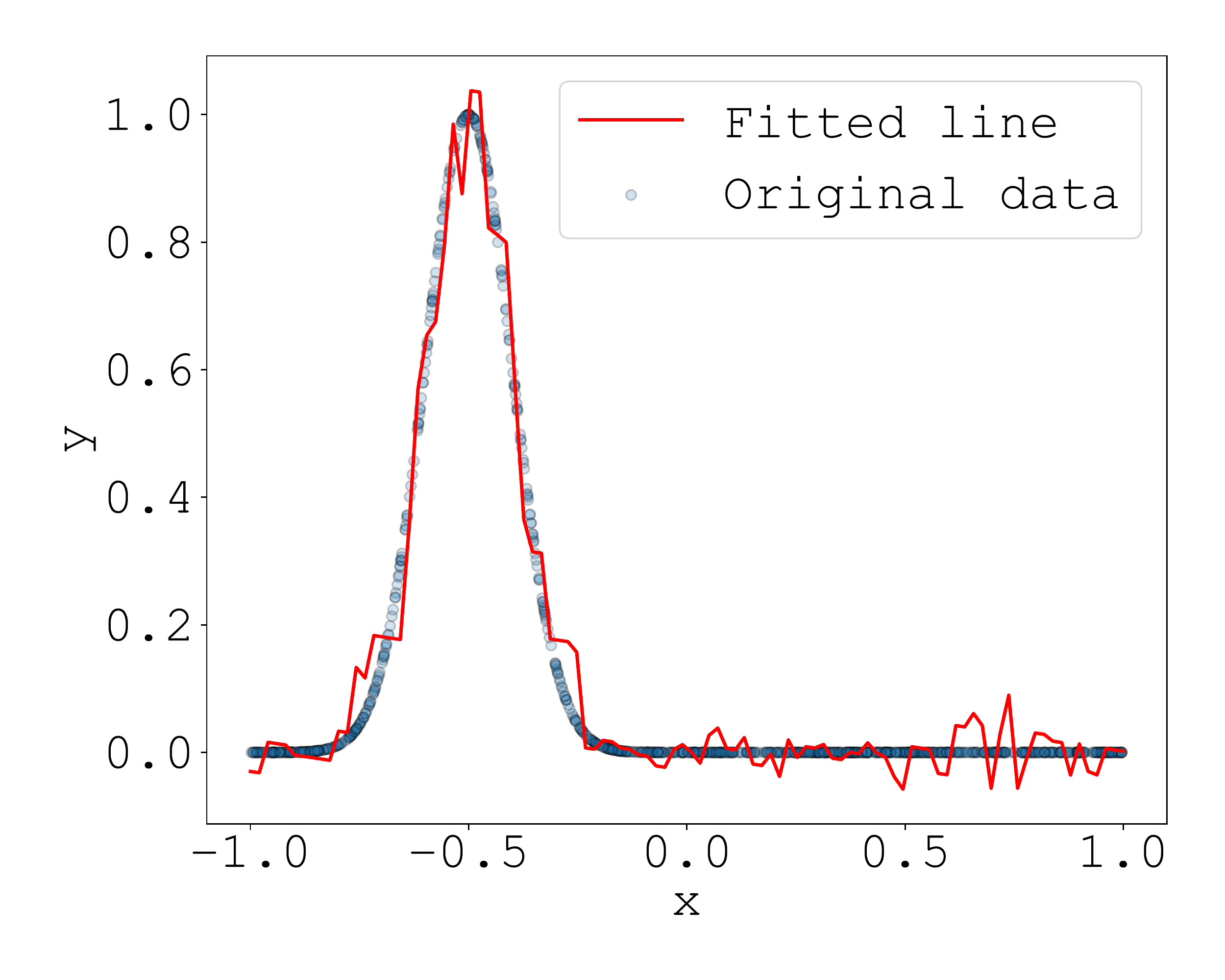}
        \end{subfigure}%
        \begin{subfigure}[c]{0.35\textwidth}
                \includegraphics[width=\linewidth]{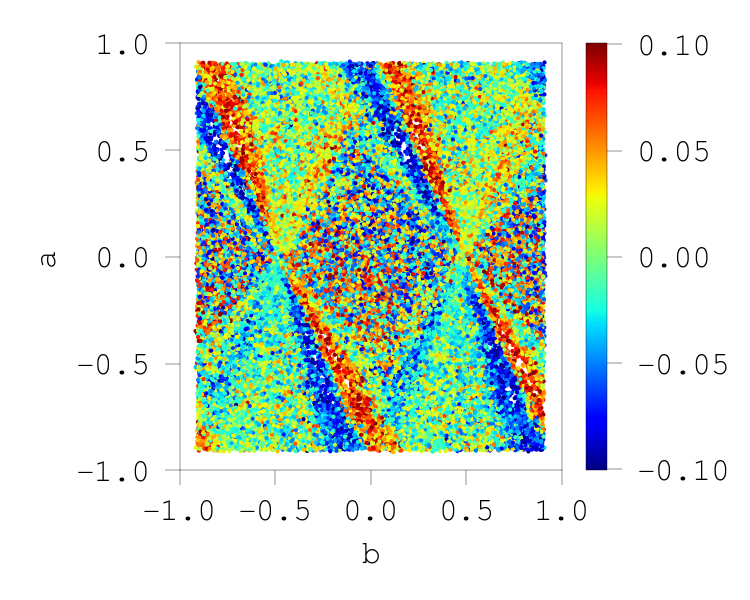}
        \end{subfigure}%
        \begin{subfigure}[c]{0.35\textwidth}
                \includegraphics[width=\linewidth]{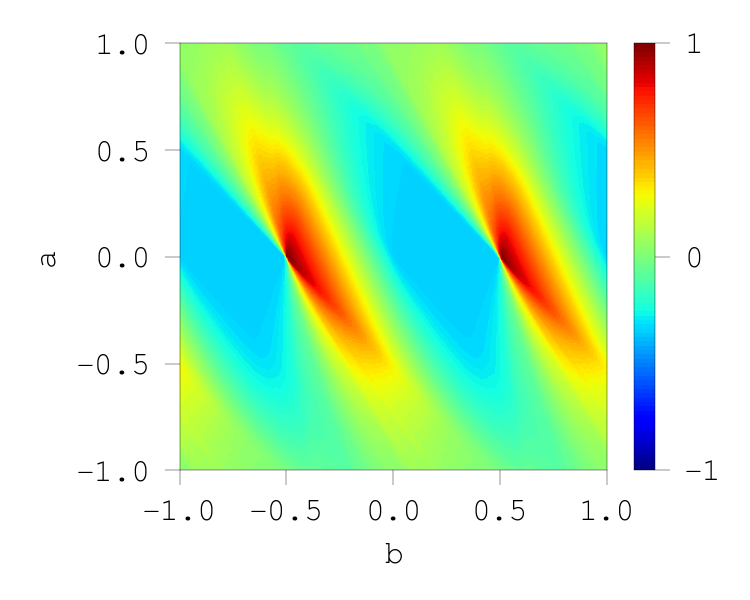}%
        \end{subfigure}\\
        \begin{subfigure}[c]{0.3\textwidth}
                \includegraphics[width=\linewidth]{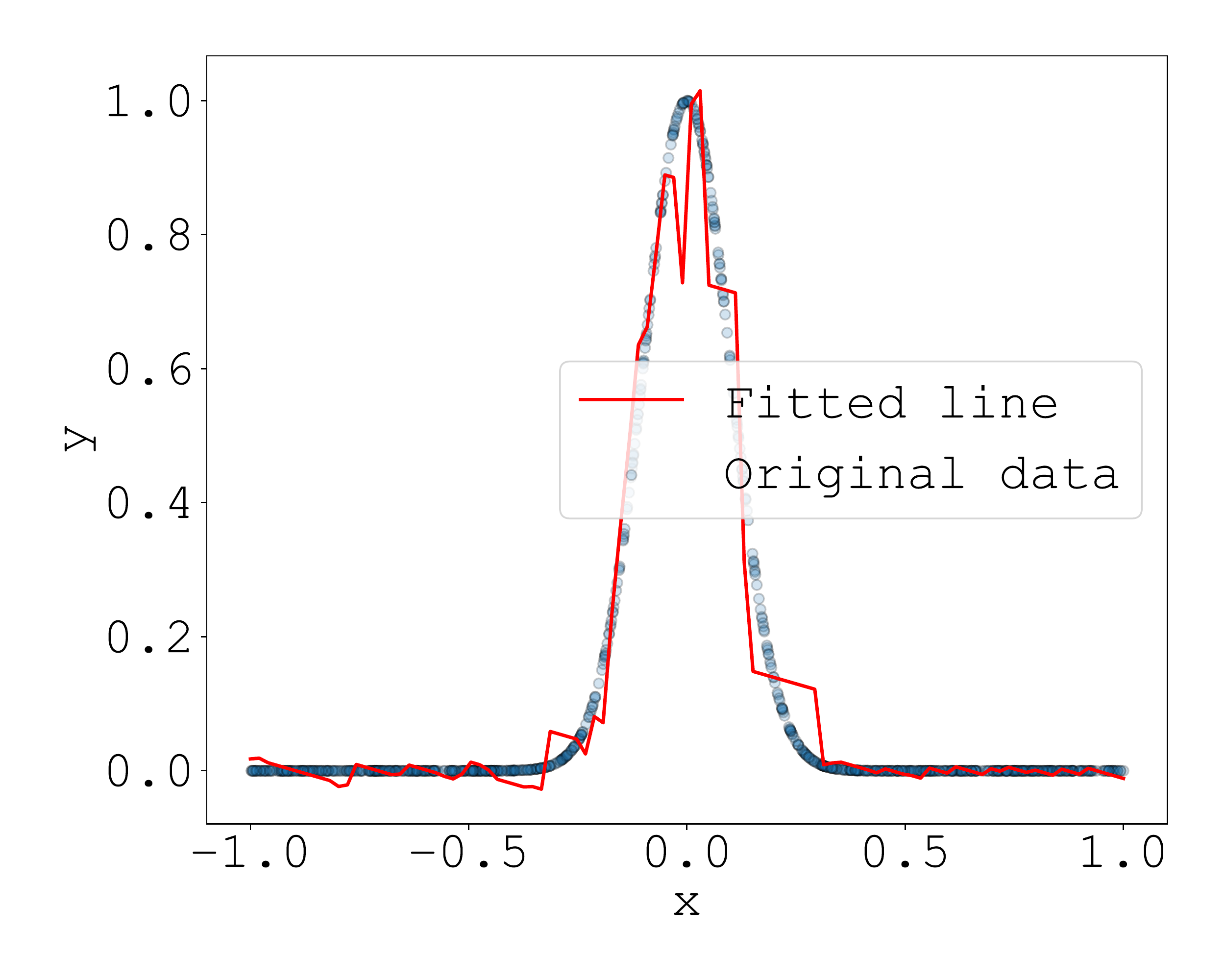}
        \end{subfigure}%
        \begin{subfigure}[c]{0.35\textwidth}
                \includegraphics[width=\linewidth]{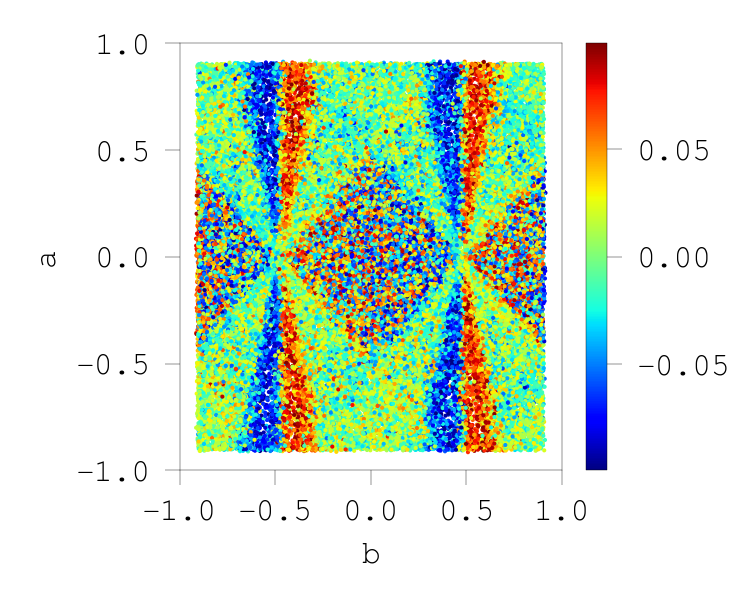}
        \end{subfigure}%
        \begin{subfigure}[c]{0.35\textwidth}
                \includegraphics[width=\linewidth]{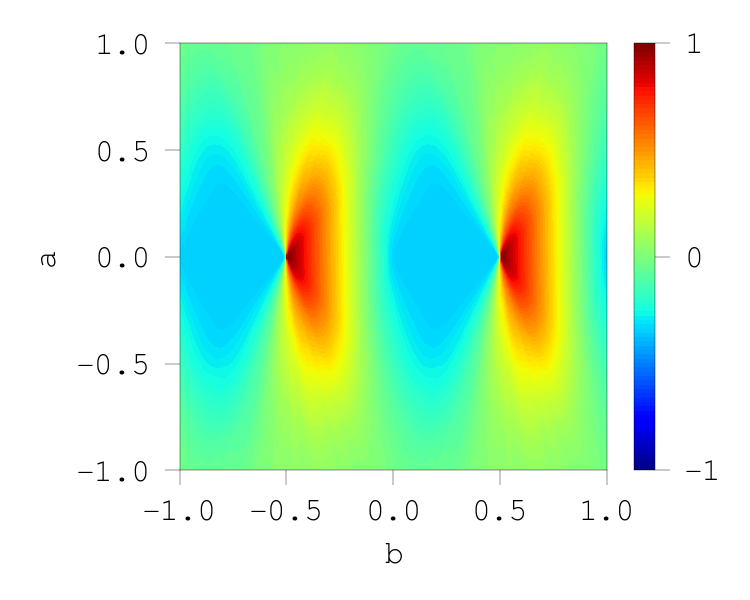}
        \end{subfigure}\\
        \begin{subfigure}[c]{0.3\textwidth}
                \includegraphics[width=\linewidth]{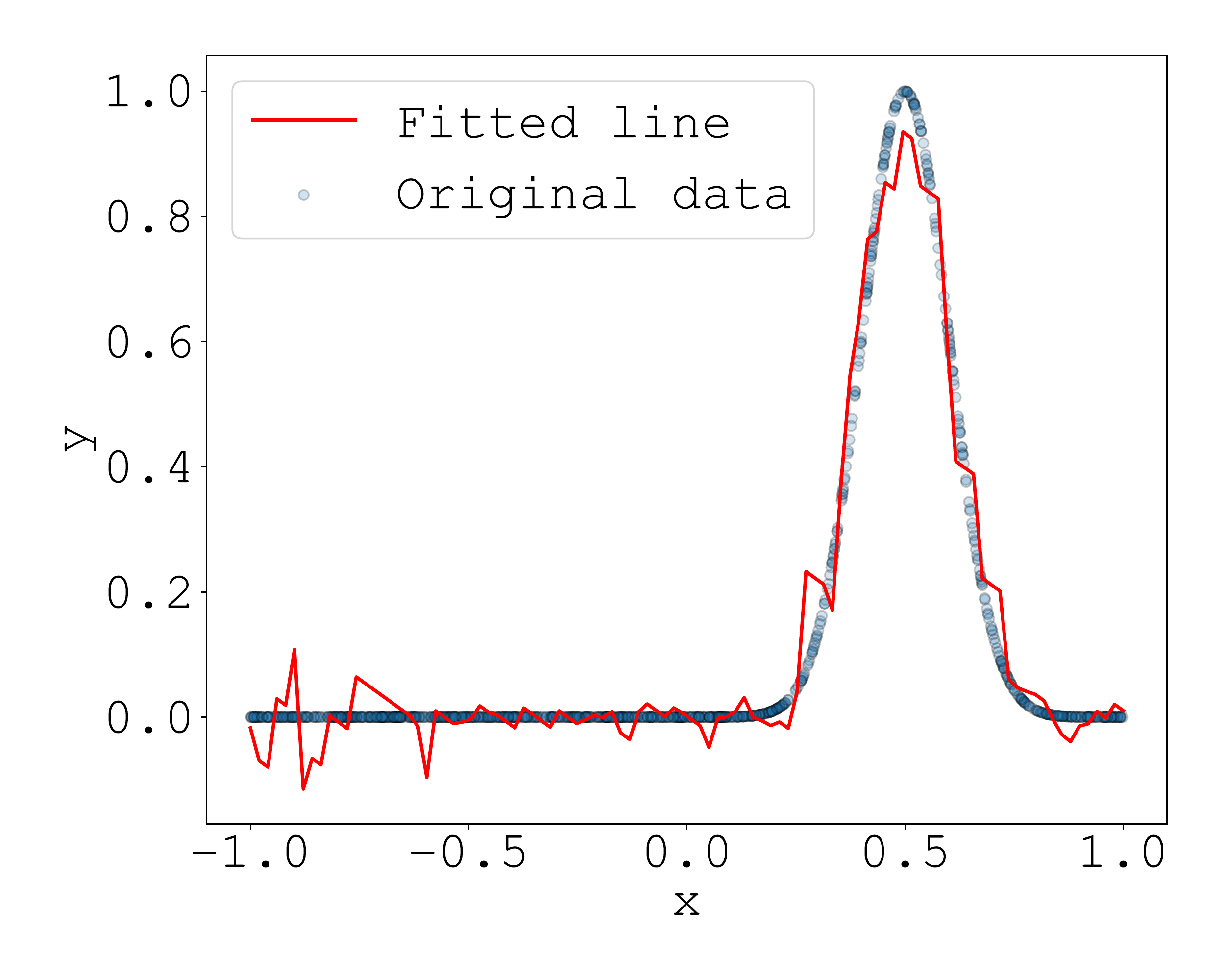}
        \end{subfigure}%
        \begin{subfigure}[c]{0.35\textwidth}
                \includegraphics[width=\linewidth]{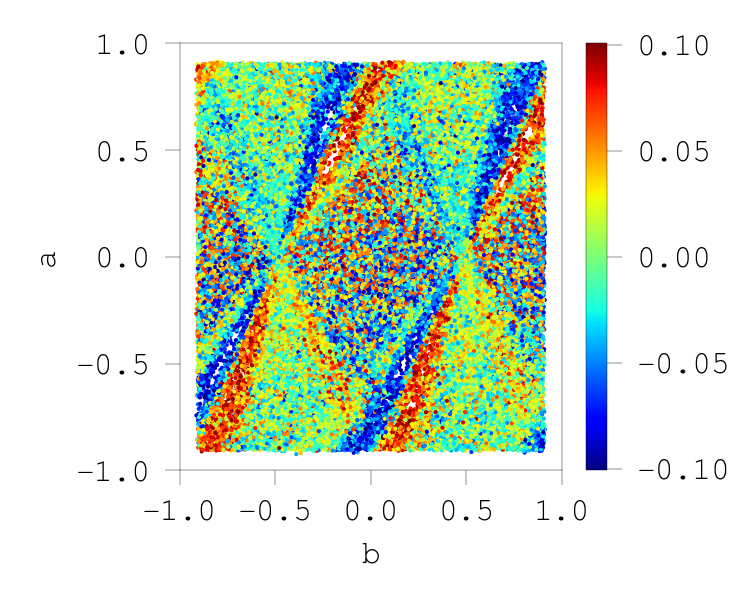}
        \end{subfigure}%
        \begin{subfigure}[c]{0.35\textwidth}
                \includegraphics[width=\linewidth]{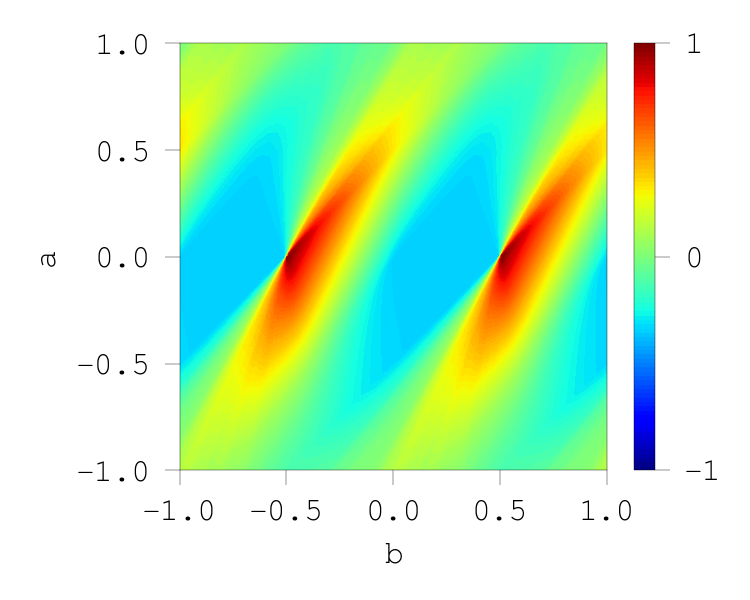}
        \end{subfigure}
        \caption{$f(x) = \exp (-|x - \mu|^2/2), (\mu=-0.5, 0.0, +0.5)$ (from top to bottom), $\sigma(z) = \relu (z)$}\label{fig:rbf}
\end{figure}

\newpage
\paragraph{Experiment~3.}
In order to see the effect of the discontinuity, we conduct the experiment on the square wave $y_i = \sign \circ \sin 2 \pi x_i$ with ReLU on $\TT = [-1/2,1/2)$.
According to the ridgelet transform, if the function has a point singularity, then the spectrum has a line singularity:
\begin{align}
    R[ \delta_{x_0} ](a,b) = \int_{\RR^m} \delta_{x_0}(x) \overline{\rho( a \cdot x - b )} \dd x = \overline{\rho( a \cdot x_0 - b )}.
\end{align}
We can clearly observe a few lines in the scatter plot.
We trained $s=1000$ networks, each single network has $d=100$ hidden units,  and the weight decay rate and learning rate were set to $\beta = 0.001$ and $\eta = 0.01$ respectively.

\begin{figure}[H]
        \centering
        \begin{subfigure}[c]{0.3\textwidth}
                \includegraphics[width=\linewidth]{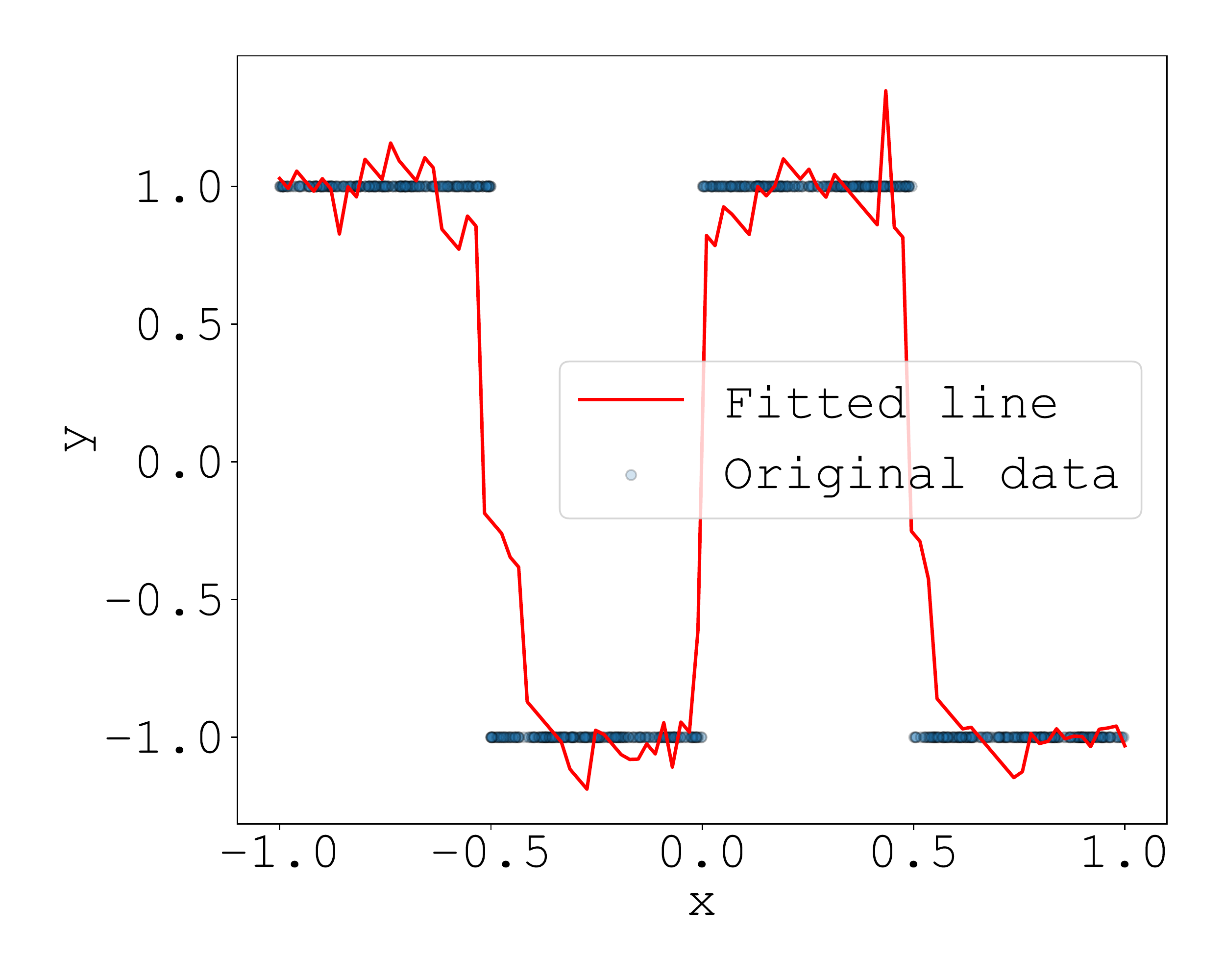}
        \end{subfigure}%
        \begin{subfigure}[c]{0.35\textwidth}
                \includegraphics[width=\linewidth]{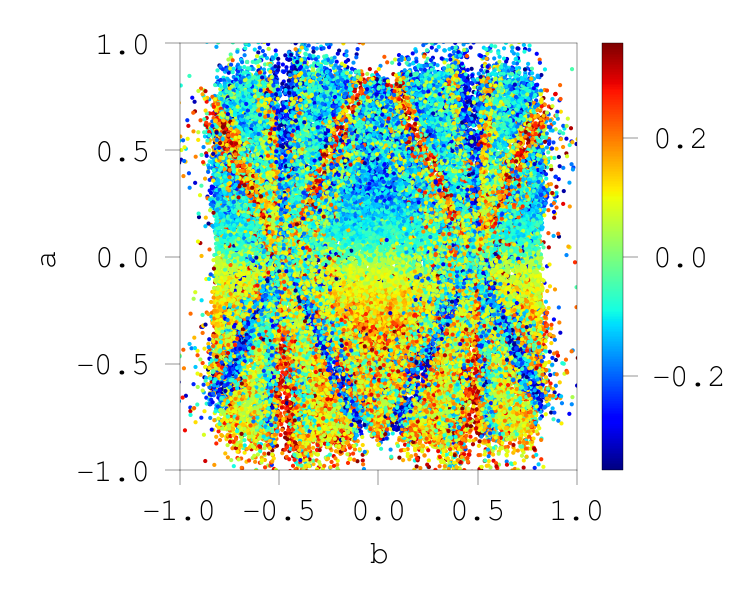}
        \end{subfigure}%
        \begin{subfigure}[c]{0.35\textwidth}
                \includegraphics[width=\linewidth]{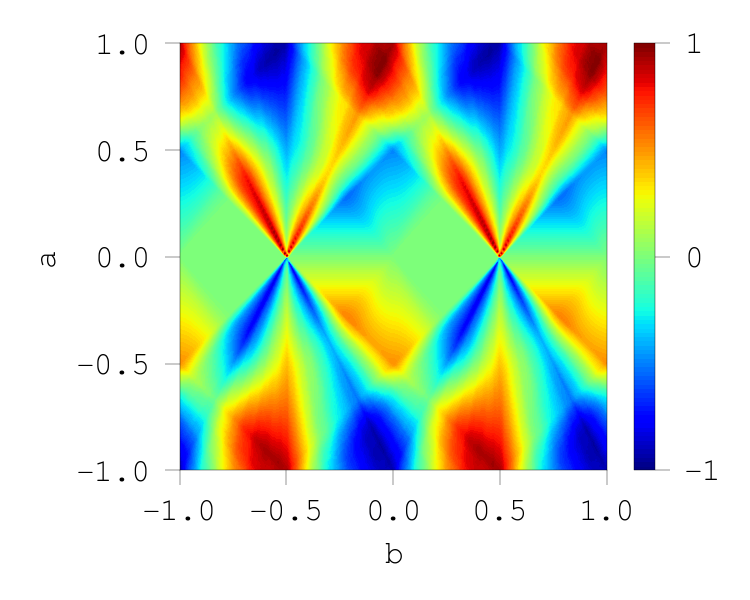}
        \end{subfigure}
        \caption{$f(x) = \sign \circ \sin (2 \pi x)$, $\sigma(z) = \relu (z)$, $\beta = 0.01$}\label{fig:sq prelu}
\end{figure}

\paragraph{Experiment~4.}
In order to see the dependence in the high-frequency, we conduct the experiment on topologist's sine curve: $y_i = \sin (2 \pi/ x_i)$, which contains an infinitely wide range of frequencies, with ReLU on $\TT = [-1/2,1/2)$. We used $n=10,000$ datapoints and $d=100$ hidden units.
As we have seen in Experiments 2 and 3, any local changes in the real domain causes a line singularity in the spectrum.
We can see dense lines in the scatter plot.
We trained $s=1000$ networks, each single network has $d=100$ hidden units, and the weight decay rate and learning rate were set to $\beta = 0.001$ and $\eta = 0.01$ respectively.

\begin{figure}[H]
        \centering
        \begin{subfigure}[c]{0.3\textwidth}
                \includegraphics[width=\linewidth]{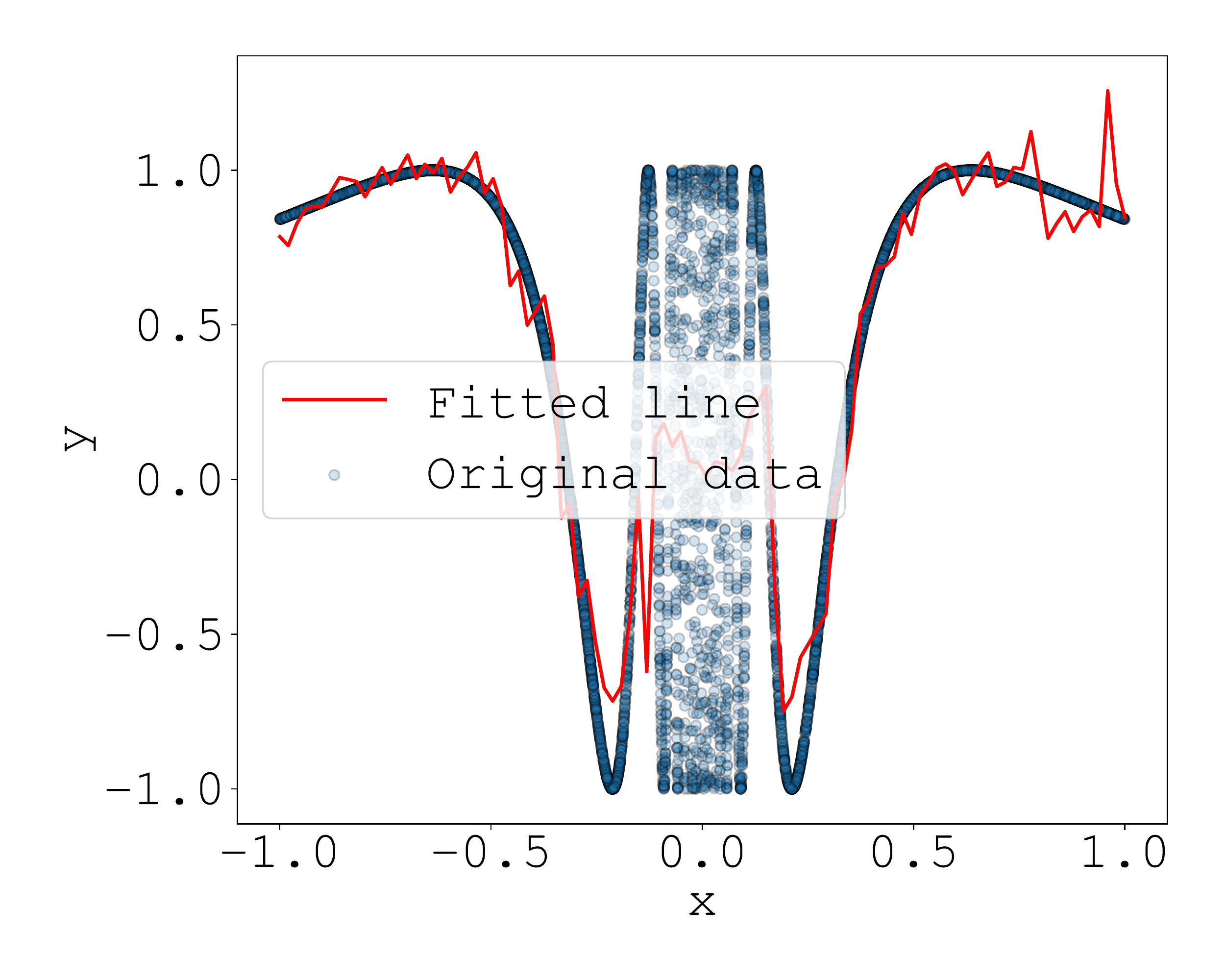}
        \end{subfigure}%
        \begin{subfigure}[c]{0.35\textwidth}
                \includegraphics[width=\linewidth]{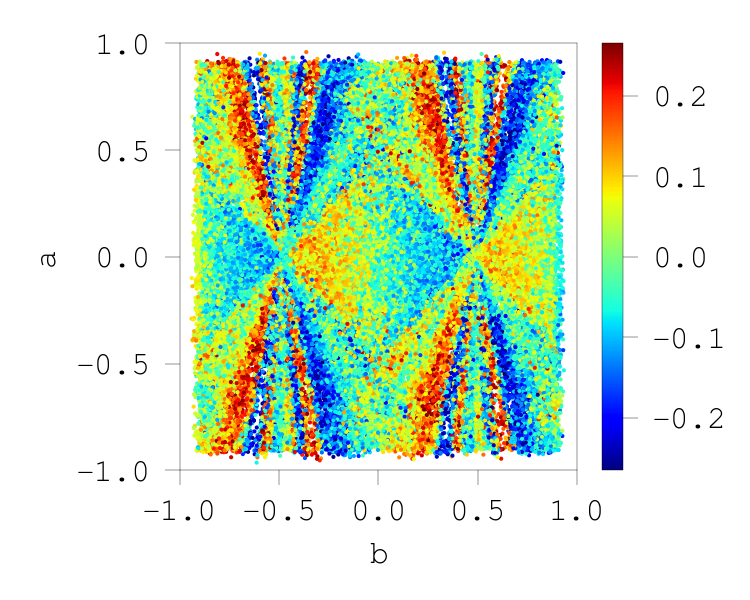}
        \end{subfigure}%
        \begin{subfigure}[c]{0.35\textwidth}
                \includegraphics[width=\linewidth]{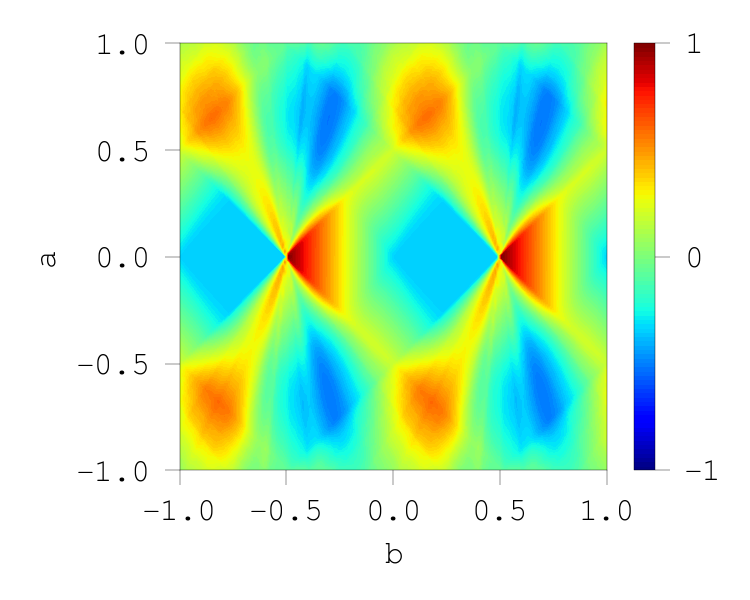}%
        \end{subfigure}
        \caption{$f(x) = \sin (2 \pi / x)$, $\sigma(z) = \relu (z)$, $\beta = 0.01$}\label{fig:topsin prelu}
\end{figure}